\documentclass{article}

\usepackage{times}
\usepackage{graphicx} %
\usepackage{subfigure}
\usepackage{natbib}
\usepackage{algorithm}
\usepackage{algorithmic}

\usepackage{latexsym, amsmath,amssymb, amsthm}
\usepackage{amsmath, amsthm, amssymb}
\usepackage{amsmath}
\usepackage{graphicx}
\usepackage{color}
\usepackage{multicol}
\usepackage{xspace}
\usepackage[colorinlistoftodos,bordercolor=orange,backgroundcolor=orange!20,linecolor=orange,textsize=scriptsize]{todonotes}
\usepackage{array}
\usepackage{booktabs}
\usepackage{multirow}

 \usepackage{hyperref}

\usepackage[accepted]{icml2015} 

\icmltitlerunning{%
Adding vs. Averaging in Distributed Primal-Dual Optimization
}

\graphicspath{ {./}{./figs/} }

\newcommand{\cocoa}{\textsc{CoCoA}\xspace} 
\newcommand{\cocoap}{\textsc{CoCoA$\!^{\bf \textbf{\footnotesize+}}$}\xspace}

\newcommand{\localSDCA}{\textsc{LocalSDCA}\xspace}

\newcommand{\Exp}{\mathbb{E}}                      %
\newcommand\tagthis{\addtocounter{equation}{1}\tag{\theequation}}

\DeclareMathOperator{\dom}{dom}         %

\DeclareMathOperator*{\argmax}{arg\,max}

\newcommand{\eqdef}{:=}
\newcommand{\R}{\mathbb{R}}                      %
\newcommand{\N}{n}                               %

\newcommand{\gap}{G}

\newcommand{\xv}{ {\bf x}}

\newcommand{\uv}{ {\bf u}}

\newcommand{\wv}{ {\bf w}}
\newcommand{\alphav}{ {\boldsymbol \alpha}}
\newcommand{\zetav}{ {\boldsymbol \zeta}}

\newcommand{\ev}{ {\bf e}}

\newcommand{\0}{ {\bf 0}}

\newcommand{\aggpar}{\gamma}

\newcommand{\vsubset}[2]{#1_{[#2]}}

\newcommand{\vc}[2]{#1^{(#2)}}                   %

\newcommand{\removed}[1]{}

\newcommand{\bP}{\mathcal{P}}
\newcommand{\bD}{\mathcal{D}}

\newcommand{\Ggk}{\mathcal{G}^{\sigma'}_k\hspace{-0.08em}}

\newcolumntype{R}[2]{%
    >{\adjustbox{angle=#1,lap=\width-(#2)}\bgroup}%
    l%
    <{\egroup}%
}

\theoremstyle{plain}
\newtheorem{theorem}{Theorem}
\newtheorem{lemma}[theorem]{Lemma}
\newtheorem{assumption}{Assumption}
\newtheorem{remark}[theorem]{Remark}
\newtheorem{corollary}[theorem]{Corollary}
\theoremstyle{definition}
\newtheorem{definition}[theorem]{Definition}

\begin{document} 

\twocolumn[
\icmltitle{%
Adding vs. Averaging in Distributed Primal-Dual Optimization
}

\icmlauthor{Chenxin  Ma$^*$}{chm514@lehigh.edu}
\icmladdress{Industrial and Systems Engineering, Lehigh University, USA}
\icmlauthor{Virginia Smith$^*$}{vsmith@berkeley.edu}
\icmladdress{University of California, Berkeley, USA}
\icmlauthor{Martin  Jaggi}{jaggi@inf.ethz.ch}
\icmladdress{ETH Z\"urich, Switzerland}
\icmlauthor{Michael I. Jordan}{jordan@cs.berkeley.edu}
\icmladdress{University of California, Berkeley, USA}
\icmlauthor{Peter  Richt\'arik}{peter.richtarik@ed.ac.uk}
\icmladdress{School of Mathematics, University of Edinburgh, UK}
\icmlauthor{Martin  Tak\'a\v{c}}{Takac.MT@gmail.com}
\icmladdress{Industrial and Systems Engineering, Lehigh University, USA}
$^*$Authors contributed equally.

\icmlkeywords{optimization algorithms, large-scale machine learning, distributed systems}

\vskip 0.3in
]

\begin{abstract}  
Distributed optimization methods for large-scale machine learning suffer 
from a communication bottleneck. It is difficult to reduce this bottleneck while still efficiently and accurately aggregating partial work from different machines. 
In this paper, we present a novel generalization of the recent communication-efficient primal-dual framework (\cocoa) for distributed optimization. 
Our framework, \cocoap, allows for \emph{additive} combination of local updates 
to the global parameters at each iteration, whereas previous schemes with convergence guarantees only 
allow conservative averaging. 
We give stronger (primal-dual) convergence 
rate guarantees for both \cocoa as well as our new variants, and generalize 
the theory for both methods to cover non-smooth convex loss functions. 
We provide an extensive experimental comparison that shows the markedly improved performance of \cocoap on several real-world distributed datasets, especially when scaling up the number of machines.
\end{abstract} 

\section{Introduction}
With the wide availability of large datasets that exceed the storage capacity of single machines, distributed optimization methods for machine learning have become increasingly important. 
Existing methods require significant communication between workers, frequently equaling the amount of local computation (or reading of local data). As a result, distributed machine learning suffers significantly from a communication bottleneck on real world systems, where communication is typically several orders of magnitudes slower than reading data from main memory.

In this work we focus on optimization problems with empirical loss minimization structure, i.e., objectives that are a sum of the loss functions of each datapoint. This includes the most commonly used regularized variants of linear regression and classification methods.
For this class of problems, the recently proposed \cocoa 
approach \cite{Yang:2013vl,jaggi2014communication} develops a communication-efficient primal-dual
scheme that targets the communication bottleneck, allowing more computation on data-local subproblems native to 
each machine before communication. By appropriately 
choosing the amount of local computation per round, this framework allows 
one to control the trade-off between \emph{communication} and \emph{local 
computation} based on the systems hardware at hand.  
 
However, the performance of \cocoa (as well as related primal SGD-based methods) is significantly reduced by the need to average updates between all machines. As the number of machines $K$ grows, the updates get diluted and slowed by $1/K$, e.g., in the case where all machines except one would have already reached the solutions of their respective partial optimization tasks. On the other hand, if the updates are instead added, the algorithms can diverge, as we will observe in the practical experiments below.

To address both described issues, in this paper we develop a novel generalization of the local \cocoa subproblems 
assigned to each worker, making the framework more powerful in the 
following sense:
Without extra computational cost, the set of locally computed updates from 
the modified subproblems (one from each machine) can be combined 
more efficiently between machines. 
The proposed \cocoap updates can be 
aggressively \emph{added} (hence the `+'-suffix), %
which yields much faster 
convergence both in practice and in theory. This difference is particularly significant as the number of 
machines~$K$ becomes large. 
 
\subsection{Contributions}

\paragraph{Strong Scaling.}
To our knowledge, our framework is the first to 
exhibit favorable \emph{strong scaling} for the class of problems considered, as 
the number of machines~$K$ increases and the data size is kept fixed.
More precisely, while the convergence rate of \cocoa degrades 
as $K$ is increased, the stronger theoretical 
convergence rate here is -- in the worst case -- \emph{independent} of $K$. 
Our experiments in Section \ref{sec:experiments} confirm the 
improved speed of convergence.
Since the number of communicated vectors is only one per round and worker, 
this favorable scaling might be surprising. Indeed, for existing methods, splitting data among more machines generally
increases communication requirements \cite{Shamir:2014tp}, which
can severely affect overall runtime.

\vspace{-0.5mm}
\paragraph{Theoretical Analysis of Non-Smooth Losses.}
While the existing analysis for \cocoa in \citep{jaggi2014communication} only 
covered smooth loss functions, here we extend the class of functions where 
the rates apply, additionally covering, e.g., Support Vector Machines and non-smooth 
regression variants.
We provide a primal-dual convergence rate for both \cocoa as well as our 
new method \cocoap in the case of general convex ($L$-Lipschitz) losses.

\vspace{-0.5mm}
\paragraph{Primal-Dual Convergence Rate.}
Furthermore, we additionally strengthen the rates by showing stronger 
primal-dual convergence for both algorithmic frameworks, which are almost 
tight to their objective-only counterparts. Primal-dual rates for \cocoa had not previously been analyzed in the general convex case.
Our primal-dual rates allow efficient and practical certificates for the 
optimization quality, e.g., for stopping criteria. The new rates apply to both 
smooth and non-smooth losses, and for both \cocoa as well as the extended 
\cocoap.

\vspace{-0.5mm}
\paragraph{Arbitrary Local Solvers.}
\cocoa as well as \cocoap allow the use of arbitrary local solvers on each machine.

\vspace{-0.5mm}
\paragraph{Experimental Results.}%
We provide a thorough experimental comparison with competing algorithms using several real-world distributed datasets. Our practical results confirm the strong scaling of \cocoap as the number of machines~$K$ grows, while competing methods, including the original \cocoa, slow down significantly with larger $K$. We implement all algorithms in
\textsf{\small Spark}, and our code is publicly available at: {\small \url{github.com/gingsmith/cocoa}}. %

\vspace{-0.5mm}
\subsection{History and Related Work}

While optimal algorithms for the serial (single machine) case are already well researched and understood, the literature in the distributed setting is relatively sparse. In particular, details on optimal trade-offs between computation and communication, as well as optimization or statistical accuracy, are still widely unclear.
For an overview over this currently active research field, we refer the reader to \cite{Balcan:2012tc,richtarik2013distributed,Duchi:2013te,Yang:2013vl,WrightAsynchrous14,fercoq2014fast,jaggi2014communication,Shamir:2014tp,DANE,DISCO,ALPHA} and the references therein.
We provide a detailed comparison of our proposed framework to the related work in Section \ref{sec:relatedWork}.

\vspace{-0.5mm}
\section{Setup}
We consider regularized empirical loss minimization problems of the following well-established form:\vspace{-2mm}
\begin{equation}
\label{eq:primal}
 \min_{\wv\in \R^d} 
 \left\{
 \bP(\wv)\eqdef  
  \frac1{\N} \sum_{i=1}^\N \ell_i( \xv_i^T \wv) + \frac\lambda 2 \|\wv\|^2 \right\}
\end{equation}
Here the vectors 
$\{\xv_i\}_{i=1}^n \subset \R^d$ represent the training data examples, and the $\ell_i(.)$ are arbitrary convex real-valued loss functions (e.g., hinge loss), possibly depending on label information for the $i$-th datapoints.  The constant $\lambda>0$ is the regularization parameter. 

The above class includes many standard problems of wide interest in machine learning, statistics, and signal processing, including support vector machines, regularized %
linear and logistic regression, %
ordinal regression, and others.

\vspace{-0.5mm}
\paragraph{Dual Problem, and Primal-Dual Certificates.}
The conjugate dual of \eqref{eq:primal} takes following form:\vspace{-1mm}
\begin{equation}
\label{eq:dual}
\max_{\alphav \in \R^\N}
  \Bigg\{
 \bD(\alphav )\eqdef  
 -\frac1n \sum_{j=1}^\N \ell_j^*(- \alpha_j)
 -\frac{\lambda}{2} 
  \left\|\frac{A \alphav}{\lambda\N}  \right\|^2  \Bigg\}\vspace{-0.5mm}
\end{equation}
Here the data matrix $A=[\xv_1, \xv_2, \dots, \xv_n] \in \R^{d\times n}$ collects all data-points as its columns, 
and $\ell_j^*$ is the conjugate function to $\ell_j$. See, e.g., \cite{ShalevShwartz:2013wl} for several concrete applications.

It is possible to assign for any dual vector $\alphav \in \R^n$ 
a corresponding primal feasible point
\begin{equation}
\label{eq:PDMapping}
\wv(\alphav)
 = \tfrac1{\lambda n} A \alphav
\end{equation}
The duality gap function is then given by:
\begin{align}
\label{eq:gap}
\gap(\alphav)
 := \bP(\wv(\alphav))-\bD(\alphav)
\end{align}
By weak duality, every value $\bD(\alphav)$ at a dual candidate~$\alphav$ provides a lower bound on every primal value $\bP(\wv)$. The duality gap is therefore a certificate on the approximation quality: The distance to the unknown true optimum $\bP(\wv^*)$ must always lie within the duality gap, i.e., $\gap(\alphav) = \bP(\wv)-\bD(\alphav) \ge \bP(\wv) - \bP(\wv^*) \ge 0$.

In large-scale machine learning settings like those considered here, the availability of such a computable measure of approximation quality is a significant benefit during training time. Practitioners using classical primal-only methods such as SGD have no means by which to accurately detect if a model has been well trained, as $P(\wv^*)$ is unknown.

\paragraph{Classes of Loss-Functions.}
To simplify presentation, we assume %
that all loss functions $\ell_i$ are non-negative, and  \vspace{-2mm}
\begin{equation}
 \ell_i(0)\leq 1  \qquad \forall i 
\label{eq:afswfevfwaefa}
\end{equation}
\begin{definition}[$L$-Lipschitz continuous loss]
A function $\ell_i: \R \to \R$ is $L$-Lipschitz continuous if $\forall a,b \in \R$, we have
\begin{equation}
 | \ell_i(a) - \ell_i(b) | \leq L |a-b|
\end{equation}
\end{definition}
\begin{definition}[$(1/\mu)$-smooth loss]
A function $\ell_i: \R \to \R$ is $(1/\mu)$-smooth  
if it is differentiable and its derivative is $(1/\mu)$-Lipschitz continuous, i.e.,
  $\forall a,b \in \R$, we have
\begin{equation}
 | \ell_i'(a) - \ell_i'(b) | \leq \frac1{\mu} |a-b|
\end{equation}
\end{definition}

\section{The \cocoap Algorithm Framework}

In this section we present our novel \cocoap framework. $\cocoap$ inherits the many benefits of CoCoA as it remains a highly flexible and scalable, communication-efficient framework for distributed optimization. $\cocoap$ differs algorithmically in that we modify the form of the local subproblems \eqref{eq:subproblem} to allow for more aggressive additive updates (as controlled by $\aggpar$). We will see that these changes allow for stronger convergence guarantees as well as  improved empirical performance. Proofs of all statements in this section are given in the supplementary material.

\paragraph{Data Partitioning.}

We write $\{\mathcal{P}_k
\}_{k=1}^K$ for the %
given partition of the datapoints $[n]\eqdef \{1,2,\dots,n\}$ over the $K$ worker machines.
We denote the size of each part by $n_k=|\mathcal{P}_k|$.
For any $k\in[K]$
and $\alphav\in \R^n$
we use the notation
$\vsubset{\alphav}{k}\in \R^n$ for the vector\vspace{-3mm}
$$
(\vsubset{\alphav}{k})_i
 :=
 \begin{cases}
 0,&\mbox{if}\ i\notin \mathcal{P}_k,\\
 \alpha_i,&\mbox{otherwise.}
\end{cases}
$$

\paragraph{Local Subproblems in \cocoap.}

We can define a data-local subproblem of the original dual optimization problem \eqref{eq:dual}, which can be solved on machine $k$ and only requires accessing data which is already available locally, i.e., datapoints with $i\in\mathcal{P}_k$. More formally, each machine $k$ is assigned the following local subproblem, depending only on the previous shared primal vector $\wv\in\R^d$, and the change in the local dual variables~$\alpha_i$ with $i\in\mathcal{P}_k$:
\begin{equation} 
\max_{\vsubset{\Delta \alphav}{k}\in\R^{n}} \ %
\Ggk(  \vsubset{\Delta \alphav}{k}; \wv, \vsubset{\alphav}{k})
\end{equation} 
where \begin{align} 
&\Ggk(  \vsubset{\Delta \alphav}{k}; \wv, \vsubset{\alphav}{k})
\eqdef
-\frac1n\sum_{i \in \mathcal{P}_k} 
\ell_i^*(-\alpha_i - (\vsubset{\Delta \alphav}{k})_i)
\nonumber
\\
&\hspace{-2mm} 
- \frac1K 
\frac{\lambda}{2}
\|\wv\|^2
-\frac1n
\wv^T A \vsubset{\Delta \alphav}{k}
- \frac\lambda2
 \sigma'  \Big\|\frac1{\lambda n} A \vsubset{\Delta \alphav}{k}\Big\|^2
 \label{eq:subproblem}
\end{align}

\paragraph{Interpretation.}
The above definition of the local objective functions $\Ggk$ are such that they closely approximate the global dual objective $\bD$, as we vary the `local' variable~$\vsubset{\Delta \alphav}{k}$, in the following precise sense:
\begin{lemma}
\label{lem:RelationOfDTOSubproblems}
For any dual
$\alphav, \Delta \alphav %
\in \R^n$, primal $\wv = \wv(\alphav)$ and real values $\aggpar,\sigma'$ satisfying~\eqref{eq:sigmaPrimeSafeDefinition}, it holds that\vspace{-2mm}
\begin{align*}
&\bD\Big(
\alphav +\aggpar 
\sum_{k=1}^K
\vsubset{\Delta \alphav}{k}\!
\Big) \geq 
 (1-\aggpar) \bD(\alphav)\nonumber
 \end{align*} \vspace{-2.3em}
 \begin{align} \hspace{8em} + \aggpar 
 \sum_{k=1}^K 
 \Ggk(\vsubset{\Delta \alphav}{k}; \wv, \vsubset{\alphav}{k}) \vspace{-2em}
\end{align}
\end{lemma} 
\vspace{-1em}

The role of the parameter $\sigma'$ is to measure the difficulty of the given 
data partition. For our purposes, we will see that it must be chosen not 
smaller than
\begin{equation}
\label{eq:sigmaPrimeSafeDefinition}
\sigma'
\geq
\sigma'_{min}
 \eqdef
 \aggpar
 \max_{\alphav\in \R^n}
 \frac{
 \|A \alphav\|^2}{\sum_{k=1}^K \|A \vsubset{\alphav}{k}\|^2} \ 
\end{equation}
\vspace{-1em}

In the following lemma, we show that this parameter can 
be upper-bounded by $\aggpar K$, which is trivial to calculate for all values $\aggpar\in\R$. We show experimentally (Section \ref{sec:experiments}) that this safe upper bound for $\sigma'$ has a minimal effect on the overall performance 
of the algorithm. Our main theorems below show 
convergence rates dependent on $\aggpar \in [\frac1K,1]$, which we refer to as the \textit{aggregation parameter}.

\begin{lemma}\label{lem:sigmaPrimeNotBad}
The choice of $\sigma' := \aggpar K$ is valid for \eqref{eq:sigmaPrimeSafeDefinition}, i.e.,
\[
\aggpar K
\geq
\sigma'_{min} 
\]
\end{lemma}

\paragraph{Notion of Approximation Quality of the Local Solver.}

\begin{assumption}[$\Theta$-approximate solution]
\label{asm:THeta}
We assume that 
there exists $\Theta \in [0,1)$ such that 
$\forall k\in [K]$, 
the local solver at any outer iteration $t$ produces
a (possibly) randomized approximate solution $\vsubset{\Delta \alphav}{k}$,
which satisfies
\begin{align}
\label{eq:localSolutionQuality}
\Exp\big[
 \Ggk(\vsubset{\Delta \alphav^*}{k};\wv, \vsubset{\alphav}{k})
-
 \Ggk(\vsubset{\Delta \alphav}{k};\wv, \vsubset{\alphav}{k})
\big] \ 
\\ \nonumber
\leq \ \Theta
\left(
 \Ggk(\vsubset{\Delta \alphav^*}{k};\wv, \vsubset{\alphav}{k})
 -
 \Ggk({\bf 0};\wv, \vsubset{\alphav}{k})
 \right),
\end{align}
where\vspace{-2mm}
\begin{align}
\label{eq:asjfcowjfcaw}
\vsubset{\Delta \alphav^*}{k}
\in \argmax_{\Delta \alphav \in \R^n} \ 
 \Ggk(\vsubset{\Delta \alphav}{k};\wv, \vsubset{\alphav}{k}) \quad \forall k\in[K] %
 \hspace{-2.5mm}\vspace{-1mm}
\end{align}

\end{assumption} 
We are now ready to describe the \cocoap framework, shown in Algorithm \ref{alg:cocoa}.
The crucial difference compared to the existing \cocoa algorithm \cite{jaggi2014communication} is the more general local subproblem,  as defined in \eqref{eq:subproblem}, as well as the aggregation parameter $\aggpar$. 
These modifications allow the option of directly adding updates to the global vector~$\wv$.

\begin{algorithm}[h]
\caption{\cocoap Framework}
\label{alg:cocoa}

\begin{algorithmic}[1]
\STATE {\bf Input:} Datapoints $A$ distributed according to partition $\{\mathcal{P}_k\}_{k=1}^K$.
Aggregation parameter $\aggpar\!\in\!(0,1]$, 
subproblem parameter $\sigma'$ for the local subproblems
$\Ggk(  \vsubset{\Delta \alphav}{k}; \wv, \vsubset{\alphav}{k})$ for each $k\in[K]$.\\
Starting point $\vc{\alphav}{0} := \0 \in \R^n$, $\vc{\wv}{0}:=\0\in \R^d$.
\FOR {$t = 0, 1, 2, \dots $}
  \FOR {$k \in \{1,2,\dots,K\}$ {\bf in parallel over computers}}
     \STATE call the local solver, computing
     a $\Theta$-approximate solution 
     $\vsubset{\Delta \alphav}{k}$   
        of  the local subproblem \eqref{eq:subproblem} 
     \STATE update $\vsubset{\vc{\alphav}{t+1}}{k} := \vsubset{\vc{\alphav}{t}}{k} + \aggpar \, \vsubset{\Delta \alphav}{k}$
     \STATE return $\Delta \wv_k :=\frac1{\lambda n} A \vsubset{\Delta \alphav}{k}$ %
  \ENDFOR
  \STATE reduce\vspace{-7mm}
\begin{equation}\label{eq:primalGlobalUpdate}\vspace{-2mm}
\vc{\wv}{t+1}  := \vc{\wv}{t} +
  \aggpar \textstyle \sum_{k=1}^K \Delta \wv_k.
\end{equation}

\ENDFOR 
\end{algorithmic}
\end{algorithm}

\section{Convergence Guarantees}
\label{sec:convergence}

Before being able to state our main convergence results, we introduce some useful quantities and the following main lemma characterizing the effect of iterations of Algorithm~\ref{alg:cocoa}, for any chosen internal local solver.

\begin{lemma}
\label{lem:basicLemma}
Let $\ell_i^*$ be strongly\footnote{%
Note that the case of weakly convex $\ell_i^*(.)$ is explicitly allowed here as well, as the Lemma holds for the case $\mu = 0$.
} %
convex with convexity parameter
$\mu \geq 0$ with respect to the norm $\|\cdot\|$, $\forall i\in[n]$.
Then for all iterations~$t$ of Algorithm~\ref{alg:cocoa} under Assumption~\ref{asm:THeta}, and any $s\in [0,1]$, it holds that
\begin{align}
\label{eq:lemma:dualDecrease_VS_dualityGap}
&\Exp[
\bD(\vc{\alphav}{t+1})
-
\bD(\vc{\alphav}{t})
 ]
\geq
\\ \nonumber
&\qquad\qquad\qquad\qquad
\aggpar
(1-\Theta)
 \Big(
 s \gap(\vc{\alphav}{t})
-
\frac{\sigma'}{2\lambda }
\big(\frac sn \big)^2
\vc{R}{t}
 \Big), \vspace{-2mm}
\end{align}
where\vspace{-2mm}
\begin{align*}
\tagthis \label{eq:defOfR}
\vc{R}{t}&:=
-
\tfrac{ \lambda\mu n (1-s)}{\sigma' s }
 \|\vc{\uv}{t}-\vc{\alphav}{t}\|^2 
\\ \qquad \nonumber &+ 
\textstyle{\sum}_{k=1}^K   
  \| A \vsubset{  (\vc{\uv}{t}  - \vc{\alphav}{t} )}{k}\|^2,
\end{align*}
for $\vc{\uv}{t} \in\R^n$
with \vspace{-1mm}
\begin{equation}
\label{eq:defintionOfUi}
-\vc{u_i}{t}
 \in \partial \ell_i(\wv(\vc{\alphav}{t})^T \xv_i).
\end{equation}
\end{lemma}

The following Lemma provides a uniform bound on~$\vc{R}{t}$:
\begin{lemma}
\label{lemma:BoundOnR}
If $\ell_i$ are $L$-Lipschitz 
continuous for all $i\in [n]$, then\vspace{-3mm}
\begin{equation}
\label{eq:asfjoewjofa}
\forall t: 
\vc{R}{t}
\leq 4L^2 
\underbrace{\sum _{k=1}^K 
\sigma_k  n_k}_{=: \sigma}, \vspace{-2mm} %
\end{equation}
where\vspace{-1mm}
\begin{equation}
\label{eq:definitionOfSigmaK}
\sigma_k \eqdef
 \max_{\vsubset{\alphav}{k} \in \R^n}
 \frac{\|A \vsubset{\alphav}{k}\|^2}{
 \|\vsubset{\alphav}{k}\|^2}.
\end{equation}
\end{lemma}

\begin{remark}
\label{rmk:asfwaefwae}
If all data-points $\xv_i$ are normalized such that
$\|\xv_i\|\leq 1$ $\forall i\in [n]$, then
$\sigma_k \leq |\mathcal{P}_k| = n_k$.
Furthermore, if we assume that the data partition is balanced, i.e., that 
$n_k = n/K$ for all $k$, then $\sigma \le n^2/K$. This can be used to bound the constants $\vc{R}{t}$, above, as  
$
\vc{R}{t} \leq  \frac{4L^2 n^2}{K}.$
\end{remark}

\subsection{Primal-Dual Convergence for General Convex Losses}

The following theorem shows the convergence for non-smooth loss functions, in terms of objective values as well as primal-dual gap.
The analysis in \cite{jaggi2014communication} only covered the case of smooth loss functions.

\begin{theorem}
\label{thm:convergenceNonsmooth}
 
Consider Algorithm \ref{alg:cocoa} with Assumption \ref{asm:THeta}. 
Let $\ell_i(\cdot)$ be $L$-Lipschitz continuous,
and $\epsilon_\gap$ $>0$ be the desired duality gap (and hence an upper-bound on primal sub-optimality).
Then after $T$ iterations, where
\begin{align}\label{eq:dualityRequirements}
T
&\geq
T_0 + 
\max\{\Big\lceil \frac1{\aggpar (1-\Theta)}\Big\rceil,\frac
{4L^2  \sigma   \sigma'}
{\lambda n^2 \epsilon_\gap
\aggpar (1-\Theta)}\},  
\\
T_0
&\geq t_0+
\left(
\frac{2}{ \aggpar (1-\Theta) }
\left(
\frac
{8L^2  \sigma   \sigma'}
{\lambda n^2 \epsilon_\gap}
-1
\right)
\right)_+,\notag
\\
t_0 &\geq 
  \max(0,\Big\lceil \tfrac1{\aggpar (1-\Theta)}
\log(
\tfrac{
 2\lambda n^2 (\bD(\alphav^* )-\bD(\vc{\alphav}{0}))
  }{4 L^2 \sigma \sigma'}
  )
 \Big\rceil),\notag
\end{align}
we have that the expected duality gap satisfies
\[
\Exp[\bP( \wv(\overline\alphav)) - \bD(\overline \alphav) ] \leq \epsilon_\gap,
\]
at the averaged iterate
\begin{equation}\label{eq:averageOfAlphaDefinition}
\overline \alphav: = \tfrac1{T-T_0}\textstyle{\sum}_{t=T_0+1}^{T-1} \vc{\alphav}{t}. 
\end{equation}

\end{theorem}

The following corollary of the above theorem clarifies our main result: The more aggressive adding of the partial updates, as compared averaging, offers a very significant improvement in terms of total iterations needed.
While the convergence in the `adding' case becomes independent of the number of machines $K$, the `averaging' regime shows the known degradation of the rate with growing $K$, which is a major drawback of the original \cocoa algorithm. This important difference in the convergence speed is not a theoretical artifact but also confirmed in our practical experiments below for different $K$, as shown e.g. in Figure~\ref{fig:scaling_k}.

We further demonstrate below that by choosing $\aggpar$ and $\sigma'$ accordingly, we can still recover the original \cocoa algorithm and its rate.

\begin{corollary}\label{cor:convergence}
Assume that 
all datapoints $\xv_i$ are bounded as $\|\xv_i\|\leq 1$
and that 
the data partition is balanced, i.e. that 
$n_k = n/K$ for all $k$.
We consider two different possible choices of the aggregation parameter~$\aggpar$: \vspace{-1em}
\begin{itemize}
\item 
(\cocoa Averaging, $\aggpar := \frac1K$):
In this case, $\sigma':=1$ is a valid choice which satisfies 
\eqref{eq:sigmaPrimeSafeDefinition}.
Then using $\sigma \le n^2/K$ in light of Remark \ref{rmk:asfwaefwae}, we have that $T$ iterations are sufficient for primal-dual accuracy $\epsilon_\gap$, with
\begin{align*}
T
&\geq
T_0 + 
\max\{\Big\lceil \frac K{1-\Theta}\Big\rceil,\frac
{4L^2      }
{\lambda \epsilon_\gap
 (1-\Theta)}\},  
\\
T_0
&\geq t_0+
\left(
\frac{2 K}{1-\Theta}
\left(
\frac
{8L^2       }
{\lambda K \epsilon_\gap}
-1
\right)
\right)_+,
\\
t_0 &\geq 
  \max(0,\big\lceil \tfrac K{1-\Theta}
\log(
 \tfrac{2\lambda (\bD(\alphav^* )-\bD(\vc{\alphav}{0}))
  } {4 K L^2   }
  )
 \big\rceil) 
\end{align*}
Hence the more machines $K$, the more iterations are needed (in the worst case).

\item
(\cocoap Adding, $\aggpar := 1$):
In this case, the choice of $\sigma':=K$ satisfies 
\eqref{eq:sigmaPrimeSafeDefinition}.
Then using $\sigma \le n^2/K$ in light of Remark \ref{rmk:asfwaefwae}, we have that $T$ iterations are sufficient for primal-dual accuracy $\epsilon_\gap$, 
with\begin{align*}
T
&\geq
T_0 + 
\max\{\Big\lceil \frac1{1-\Theta}\Big\rceil,\frac
{4L^2  }
{\lambda  \epsilon_\gap
  (1-\Theta)}\},  
\\
T_0
&\geq t_0+
\left(
\frac{2}{1-\Theta}
\left(
\frac
{8L^2   }
{\lambda  \epsilon_\gap}
-1
\right)
\right)_+,
\\
t_0 &\geq 
  \max(0,\big\lceil \tfrac1{1-\Theta}
\log(
\tfrac{
 2\lambda n  (\bD(\alphav^* )-\bD(\vc{\alphav}{0}))
  }{4 K L^2}
  )
 \big\rceil)
\end{align*}
This is significantly better than the averaging case.
\end{itemize}
\end{corollary}
\vspace{-.5em}

In practice, we usually have $\sigma \ll n^2/K$, and hence the actual convergence rate can be much better than the proven worst-case bound.
Table \ref{tbl:Sigma}
shows that the actual value of $\sigma$ is typically between one and two orders of magnitudes smaller compared to our used upper-bound $n^2/K$.
\vspace{-1em} 

\begin{table}[h]
  \centering
  \caption{The ratio of upper-bound $\tfrac{n^2}{K}$ divided by the true value of the parameter~$\sigma$, for some real datasets. }
    \vspace{1mm}
  \label{tbl:Sigma}
  \scriptsize
    \begin{tabular}{rrrrrrr}
    \toprule
    K     & 16    & 32    & 64    & 128   & 256   & 512 \\
    \midrule
    news  & 15.483 & 14.933 & 14.278 & 13.390 & 12.074 & 10.252 \\
    real-sim & 42.127 & 36.898 & 30.780 & 23.814 & 16.965 & 11.835 \\
    rcv1  & 40.138 & 23.827 & 28.204 & 21.792 & 16.339 & 11.099 \\
     \midrule
    K     & 256   & 512   & 1024  & 2048  & 4096  & 8192 \\
    \midrule
    covtype & 17.277 & 17.260 & 17.239 & 16.948 & 17.238 & 12.729 
\\ \bottomrule
    \end{tabular} 
\end{table}%
    \vspace{-1mm}

\subsection{Primal-Dual Convergence for Smooth Losses}

The following theorem shows the convergence for smooth losses, in terms of the objective as well as primal-dual gap.

\begin{theorem}
\label{thm:convergenceSmoothCase}
Assume the loss functions functions 
$\ell_i$ are $(1/\mu)$-smooth $\forall i\in[n]$.
We define $\sigma_{\max} = 
\max_{k\in[K]} \sigma_k$. Then after $T$ iterations of Algorithm \ref{alg:cocoa}, with\vspace{-1mm}
$$
 T
    \geq 
\tfrac{1}
   {\aggpar
(1-\Theta)}
\tfrac
{\lambda\mu n+
\sigma_{\max} \sigma'}
{ \lambda\mu n }
    \log \tfrac1{\epsilon_\bD} , \vspace{-1mm}
$$
it holds that\vspace{-3mm}
$$\Exp[\bD(\alphav^*)
  - \bD(\vc{\alphav}{T})]
   \leq \epsilon_\bD.$$
Furthermore, after $T$ iterations with\vspace{-1mm}
$$
 T 
    \geq 
\tfrac{1}
   {\aggpar
(1-\Theta)}
\tfrac
{\lambda\mu n+
\sigma_{\max} \sigma'}
{ \lambda\mu n }
    \log 
\left(
\tfrac{1}
   {\aggpar
(1-\Theta)}
\tfrac
{\lambda\mu n+
\sigma_{\max} \sigma'}
{ \lambda\mu n }
    \tfrac1{\epsilon_\gap}
    \right)  ,
$$
we have the expected duality gap
$$
\Exp[
\bP( \wv(\vc{\alphav}{T})) - \bD(\vc{\alphav}{T})
]\leq \epsilon_\gap.
$$
\end{theorem}

\vspace{-2mm}
The following corollary is analogous to Corollary \ref{cor:convergence}, but for the case of smooth loses. %
It  again shows that while the \cocoa variant degrades with the increase of the number of machines $K$, the $\cocoap$ rate is independent of $K$.

\begin{corollary}\label{cor:convergenceSmooth}
Assume that 
all datapoints $\xv_i$ are bounded as $\|\xv_i\|\leq 1$
and that 
the data partition is balanced, i.e., that 
$n_k = n/K$ for all $k$.
We again consider the same two different possible choices of the aggregation parameter~$\aggpar$: \vspace{-3mm}
\begin{itemize}
\item 
(\cocoa Averaging, $\aggpar := \frac1K$):
In this case, $\sigma':=1$ is a valid choice which satisfies 
\eqref{eq:sigmaPrimeSafeDefinition}.
Then using $\sigma_{\max} \le n_k = n/K$ in light of Remark \ref{rmk:asfwaefwae}, we have that $T$ iterations are sufficient for suboptimality %
$\epsilon_\bD$, with
\begin{align*}
T
&\geq
\tfrac{1}
   {1-\Theta}
\tfrac
{\lambda\mu K +
1}
{ \lambda\mu }
    \log \tfrac1{\epsilon_\bD}
\end{align*}
Hence the more machines $K$, the more iterations are needed (in the worst case).
\vspace{-2mm}
\item
(\cocoap Adding, $\aggpar := 1$):
In this case, the choice of $\sigma':=K$ satisfies 
\eqref{eq:sigmaPrimeSafeDefinition}.
Then using $\sigma_{\max} \le n_k = n/K$ in light of Remark \ref{rmk:asfwaefwae}, we have that $T$ iterations are sufficient for suboptimality %
$\epsilon_\bD$, with
\begin{align*}
T
&\geq
\tfrac{1}{1-\Theta}
\tfrac
{\lambda\mu+1}
{ \lambda\mu }
    \log \tfrac1{\epsilon_\bD}
\end{align*}
This is significantly better than the averaging case.
Both rates hold analogously for the duality gap.

\end{itemize}
\end{corollary}

\subsection{Comparison with Original \cocoa}

\begin{remark}%

If we choose averaging
($\aggpar :=\frac1K$) for aggregating the updates, together with $\sigma' := 1$,
then the resulting Algorithm \ref{alg:cocoa} is identical to \cocoa analyzed in \cite{jaggi2014communication}. 
However, they only provide convergence for smooth loss functions $\ell_i$ and have guarantees for dual sub-optimality and not the duality gap. Formally, when $\sigma' = 1$, the subproblems \eqref{eq:subproblem} will differ from the original dual $\bD(.)$ only by an additive constant, which does not affect the local optimization algorithms used within \cocoa.
\end{remark}

\section{SDCA as an Example Local Solver}

We have shown convergence rates for Algorithm \ref{alg:cocoa}, depending solely on the approximation quality $\Theta$ of the used local solver (Assumption~\ref{asm:THeta}).
Any chosen local solver in each round receives the local $\alphav$ variables as an input, as well as a shared vector $\wv \overset{\eqref{eq:PDMapping}}{=} \wv(\alphav )$ being compatible with the last state of all global $\alphav\in \R^n$ variables.

As an illustrative example for a local solver, Algorithm \ref{alg:localSDCA} below summarizes randomized coordinate ascent (SDCA) applied on the local subproblem \eqref{eq:subproblem}.
The following two Theorems
(\ref{thm:LocalSDCA_smooth2},
\ref{thm:LocalSDCA_smooth1})
characterize the local convergence %
for both smooth and non-smooth  functions. In all the results we will use
 $r_{\max} := \max_{i \in [n]} \|\xv_i\|^2$.
 
\begin{algorithm}[h]
\caption{\localSDCA$(\wv, \vsubset{\alphav}{k}, k, H)$}
\label{alg:localSDCA}

\begin{algorithmic}[1]
\STATE {\bf Input:} 
$\vsubset{\alphav}{k}, \wv=\wv(\alphav)$
\STATE {\bf Data:} Local $\{(\xv_i, y_i)\}_{ i \in \mathcal{P}_k}$
\STATE {\bf Initialize:}   $\vc{\Delta \alphav_{[k]}}{0} := \0 \in \mathbb R^{n}$
\FOR {$h = 0, 1, \dots ,H-1$} %
  \STATE choose $i\in \mathcal{P}_k$ uniformly at random
  \STATE 
  $\displaystyle
  \delta^*_i
 := \argmax_{\delta_i \in \R} \,
 \Ggk(  
 \vc{\vsubset{\Delta \alphav}{k}}{h}
 + \delta_i \ev_i; \wv, \vsubset{\alphav}{k})$
  \STATE  
  $\vsubset{
  \Delta\alphav^{(h+1)}}{k} := \vsubset{
  \Delta\alphav^{(h)}}{k} + \delta^*_i \ev_i$
\ENDFOR 
\STATE {\bf Output:} $\Delta\alphav_{[k]}^{(H)}$ %
\end{algorithmic}
\end{algorithm}
 
\begin{theorem}
\label{thm:LocalSDCA_smooth2}
Assume the functions $\ell_i$ are $(1/\mu)-$smooth for $i\in[n]$.
Then Assumption~\ref{asm:THeta} on the  local approximation quality $\Theta$ is satisfied
for \localSDCA as given in Algorithm \ref{alg:localSDCA}, if we
choose the number of inner iterations $H$ as \vspace{-1mm}
\begin{equation}
\label{eq:asjfwjfdwafcea}
H \geq n_k \frac{\sigma' r_{\max} + \lambda n \mu}{\lambda n \mu} \log \frac1{\Theta} . \vspace{-1mm}
\end{equation}
\end{theorem}

\begin{theorem}
\label{thm:LocalSDCA_smooth1}
Assume the functions $\ell_i$ are $L$-Lipschitz for $i\in[n]$.
Then Assumption~\ref{asm:THeta} on the local approximation quality $\Theta$ is satisfied
for \localSDCA as given in Algorithm \ref{alg:localSDCA}, if we
choose the number of inner iterations $H$ as  \vspace{-1mm}
\begin{equation}
\label{eq:H_convexLoss}
H \geq   n_k   
 \bigg( 
\frac{1-\Theta}{\Theta  } 
  +
  \frac{\sigma'r_{\max}}
       {2\Theta \lambda n^2}        
\frac{\| \vsubset{\Delta \alphav^*}{k}\|^2}
{  
\Ggk( \vsubset{\Delta \alphav^*}{k};.) %
-   \Ggk( {\bf 0};.)
}
 \bigg) .
\end{equation}
\end{theorem}

\begin{remark}
Between the different regimes allowed in \cocoap (ranging between averaging and adding the updates) the computational cost for obtaining the required local approximation quality varies with the choice of $\sigma'$.
From the above worst-case upper bound, we note that the cost can increase with $\sigma'$, as aggregation becomes more aggressive. %
However, as we will see in the practical experiments in Section \ref{sec:experiments} below, the additional cost is negligible compared to the gain in speed from the different aggregation, when measured on real datasets.
\end{remark}

\section{Discussion and Related Work}
\label{sec:relatedWork}

\paragraph{SGD-based Algorithms.}
For the empirical loss minimization problems of interest here, stochastic subgradient descent (SGD) based methods are well-established.
Several distributed variants of SGD have been proposed, many of which build on the idea of a parameter server \cite{Niu:2011wo,Liu:2014wj,Duchi:2013te}. %
The downside of this approach, even when carefully implemented, is that the amount of required communication is equal to the amount of data read locally (e.g., mini-batch SGD with a batch size of 1 per worker). These variants are in practice not competitive with the more communication-efficient methods considered here, which allow more local updates per round.

\vspace{-1mm}
\paragraph{One-Shot Communication Schemes.}
At the other extreme, there are distributed methods using only a single round of communication, such as \cite{Zhang:2013wq, Zinkevich:2010tj,Mann:2009tr,McWilliams:2014tl}. %
These require additional assumptions on the partitioning of the data, and furthermore can not guarantee convergence to the optimum solution for all regularizers, as shown in, e.g., \cite{DANE}. \cite{Balcan:2012tc} shows additional relevant lower bounds on the minimum number of communication rounds necessary for a given approximation quality for similar machine learning problems.

\vspace{-1mm}
\paragraph{Mini-Batch Methods.} Mini-batch methods are more flexible and lie within these two communication vs. computation extremes. However,
mini-batch versions of both SGD and coordinate descent (CD) \cite{richtarik2013distributed,MinibatchASDCA,Yang:2013vl, ALPHA, QUARTZ} suffer from their convergence rate degrading towards the rate of batch gradient descent as the size of the mini-batch is increased. 
This follows because mini-batch updates are made based on the outdated previous parameter vector $\wv$, in contrast to methods that allow immediate local updates like \cocoa.
Furthermore, the aggregation parameter for mini-batch methods is harder to tune, as it can lie anywhere in the order of mini-batch size.
In the \cocoa setting, the parameter lies in the smaller range given by $K$. 
Our \cocoap extension avoids needing to tune this parameter entirely, by adding.

\newcommand{\smalltrimfig}[1]{\subfigure{\includegraphics[trim = 30 180 30 180, clip, width=.246\linewidth]{#1}}}

\begin{figure*}[t!]
\smalltrimfig{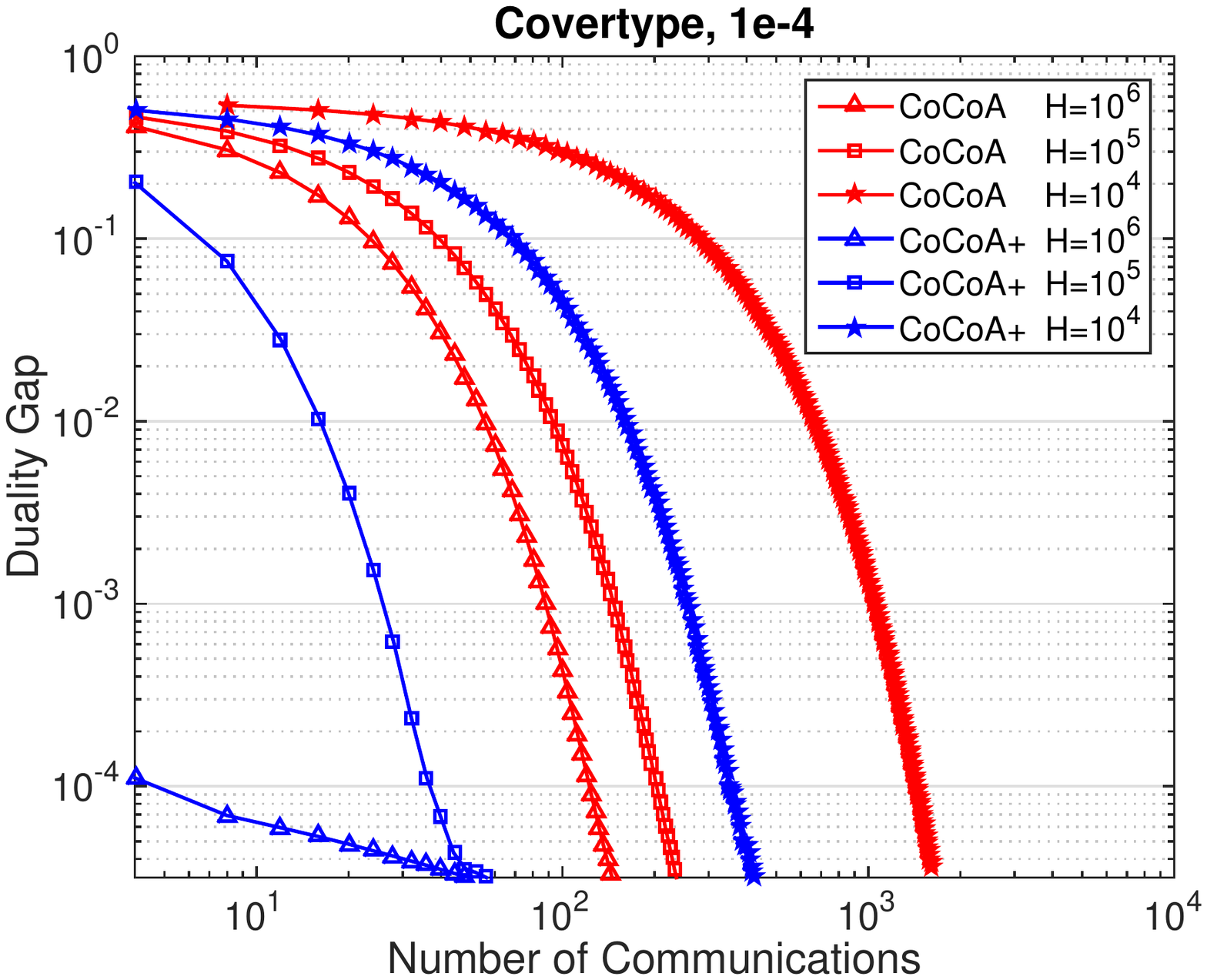}  \vspace{-1em}
\smalltrimfig{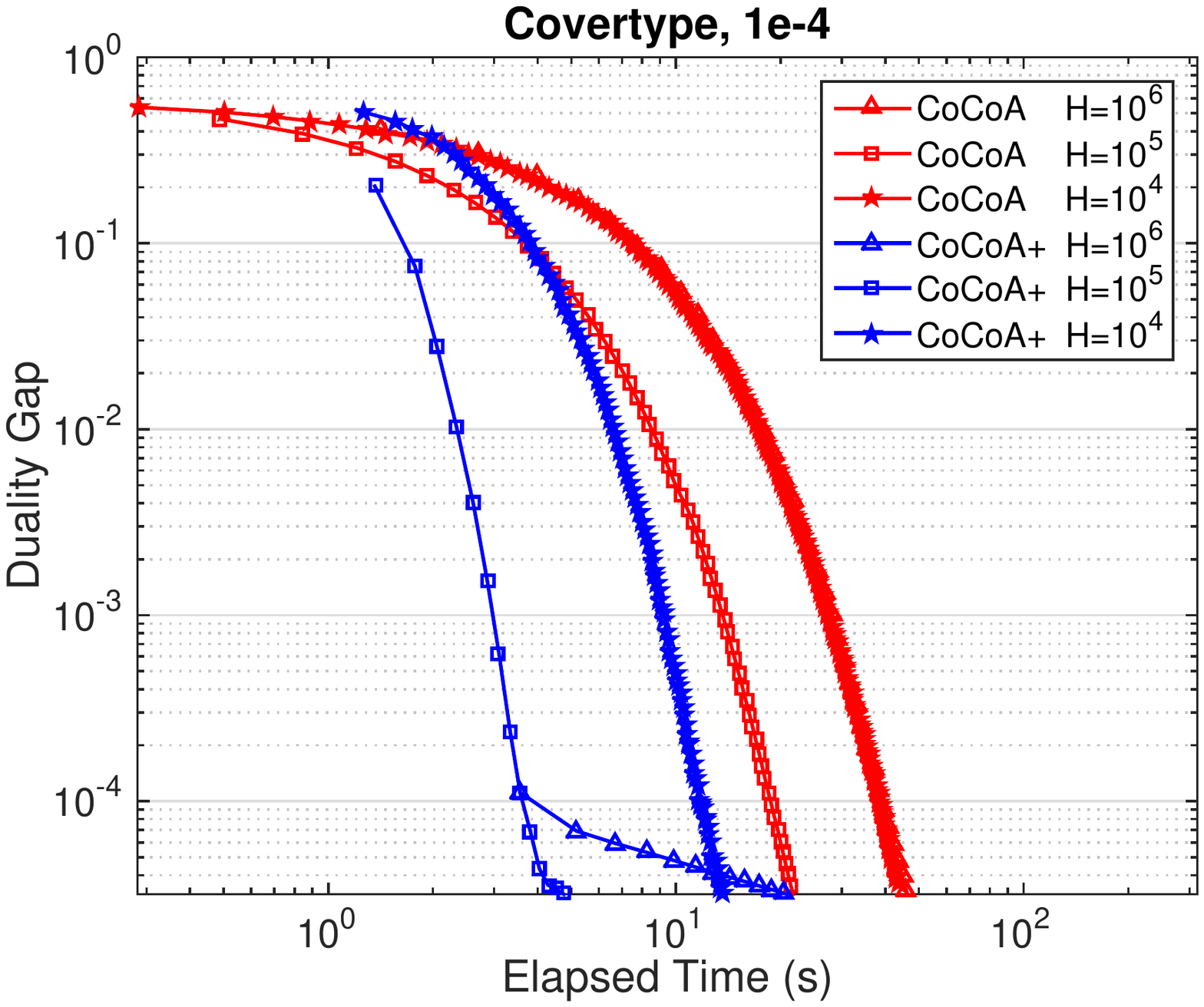}
\smalltrimfig{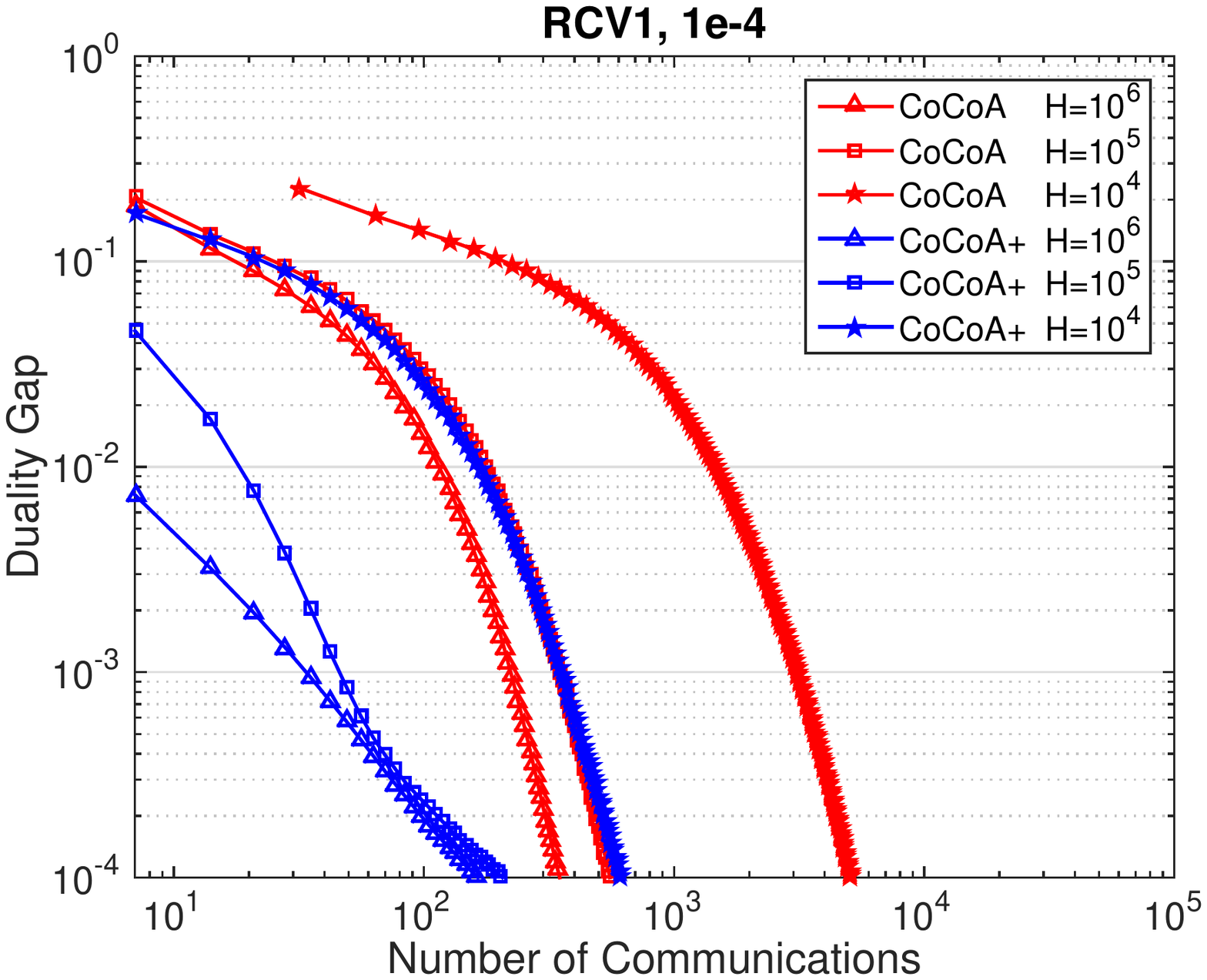}
\smalltrimfig{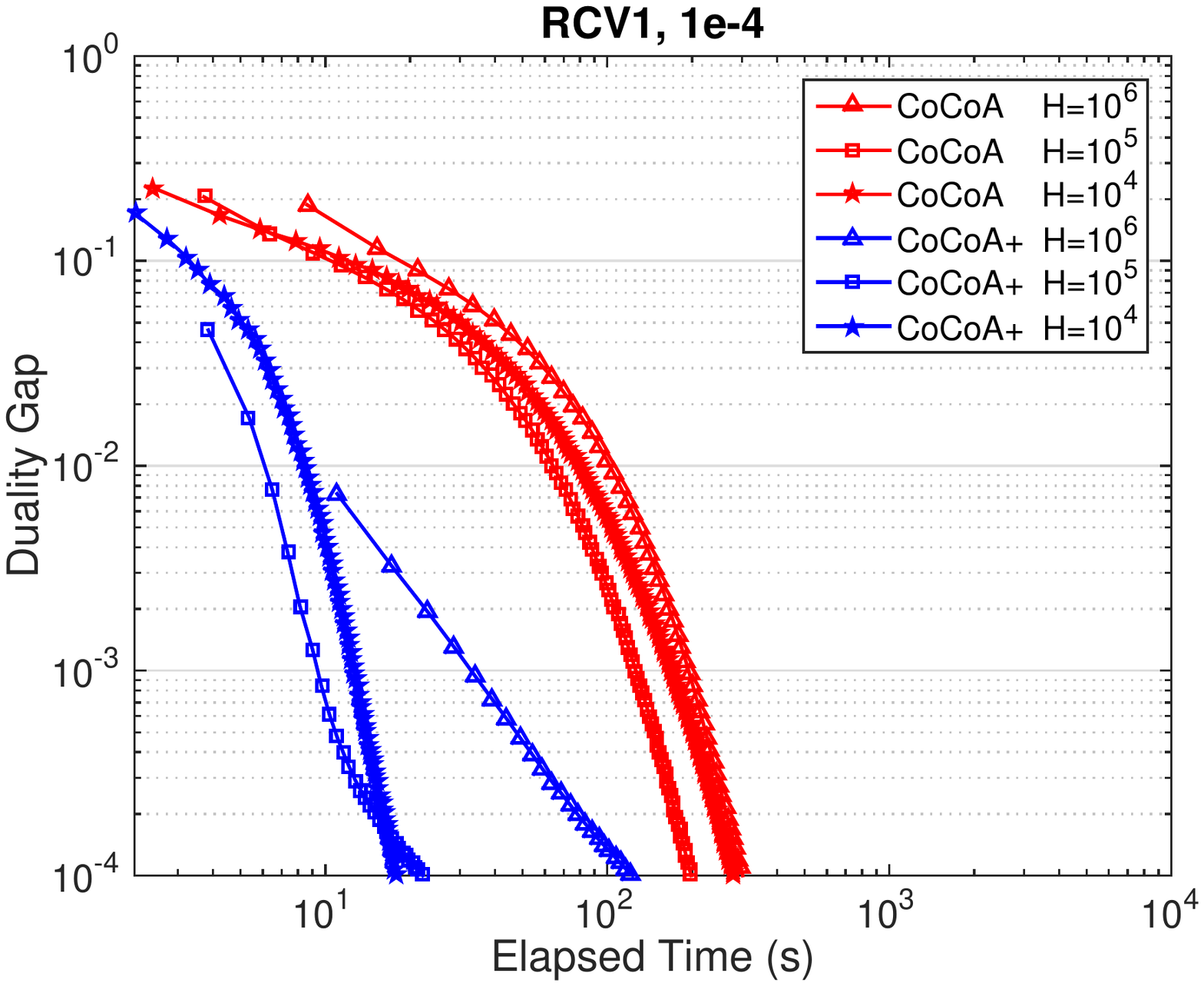} \vspace{-1em}
\smalltrimfig{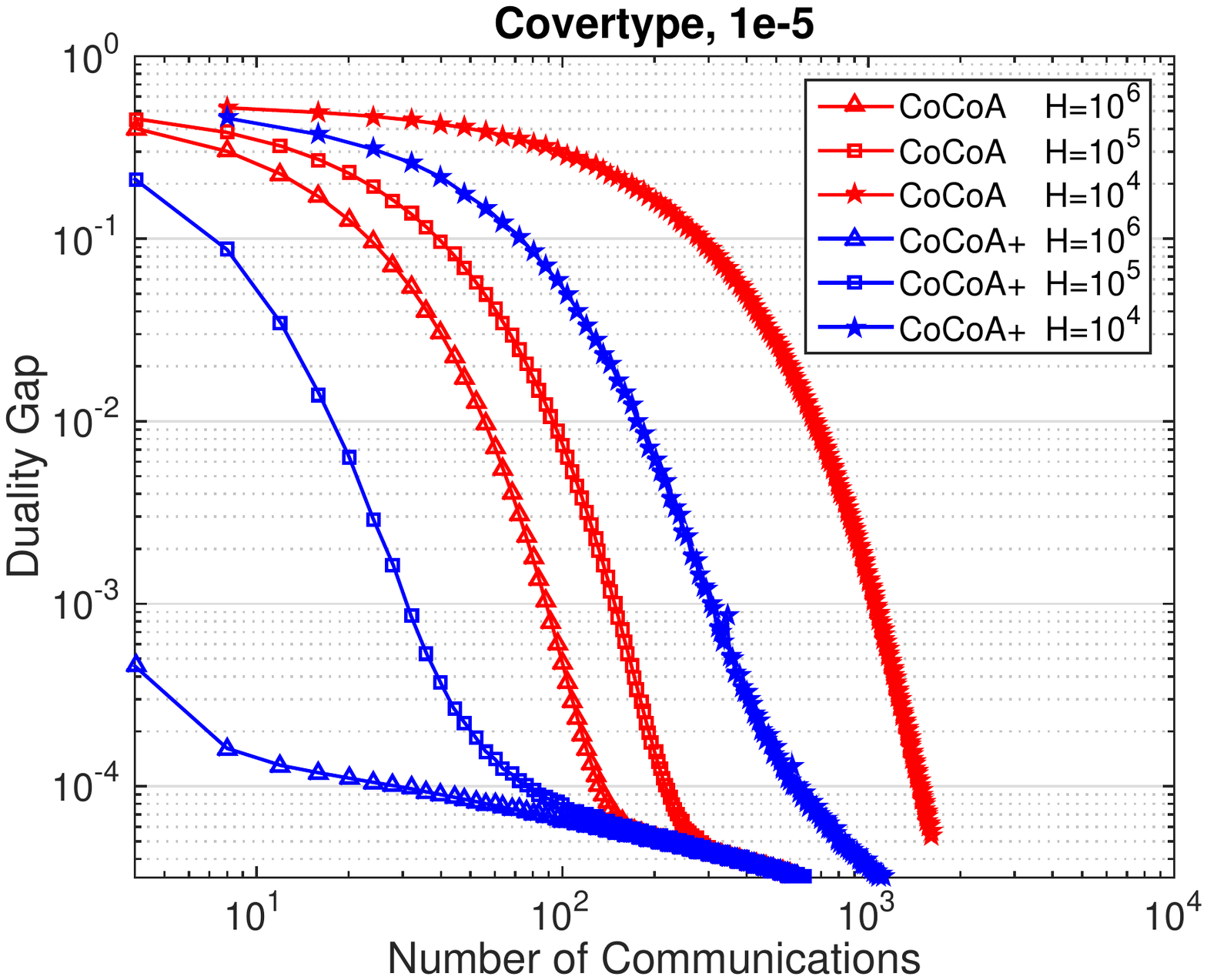}
\smalltrimfig{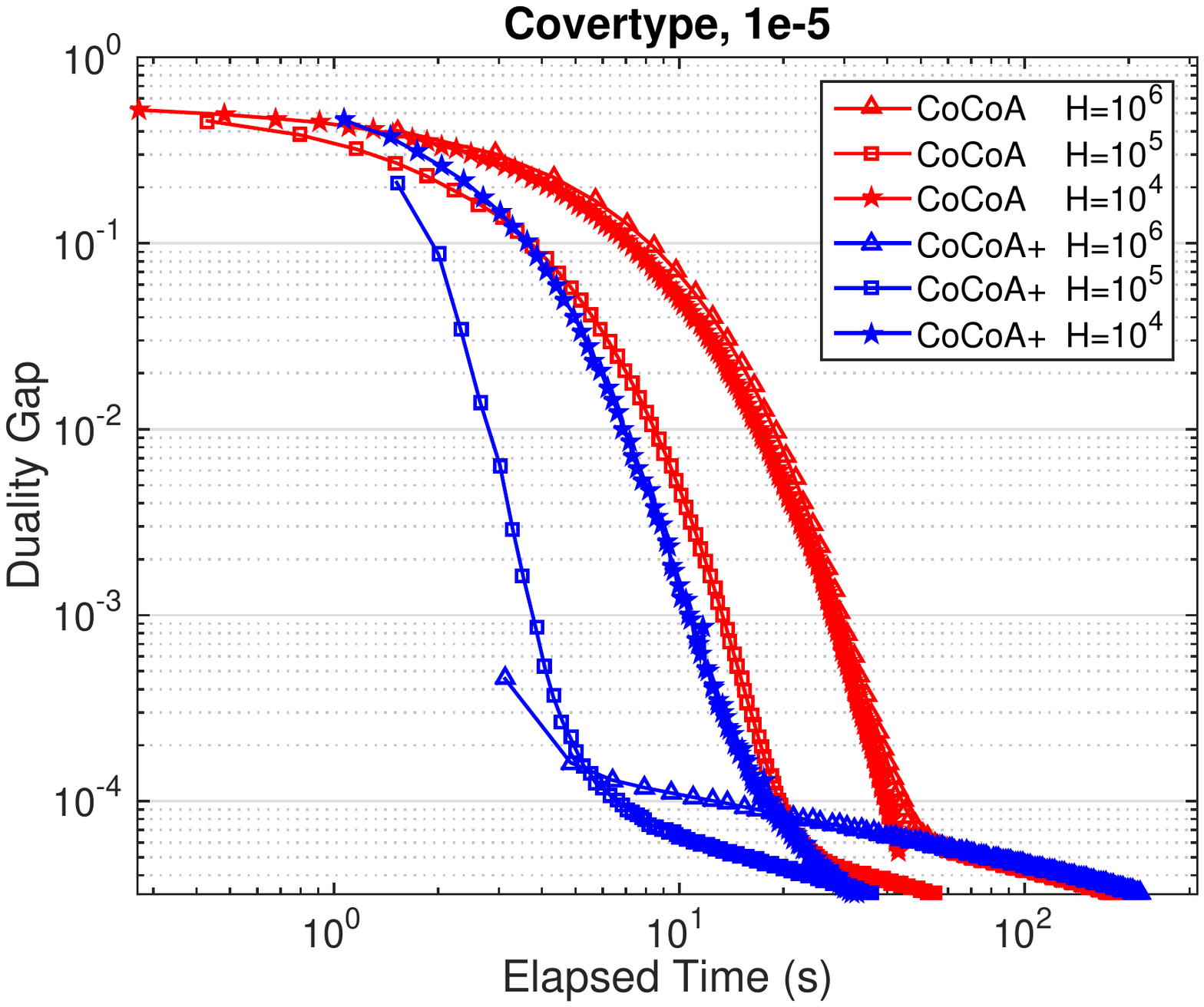}
\smalltrimfig{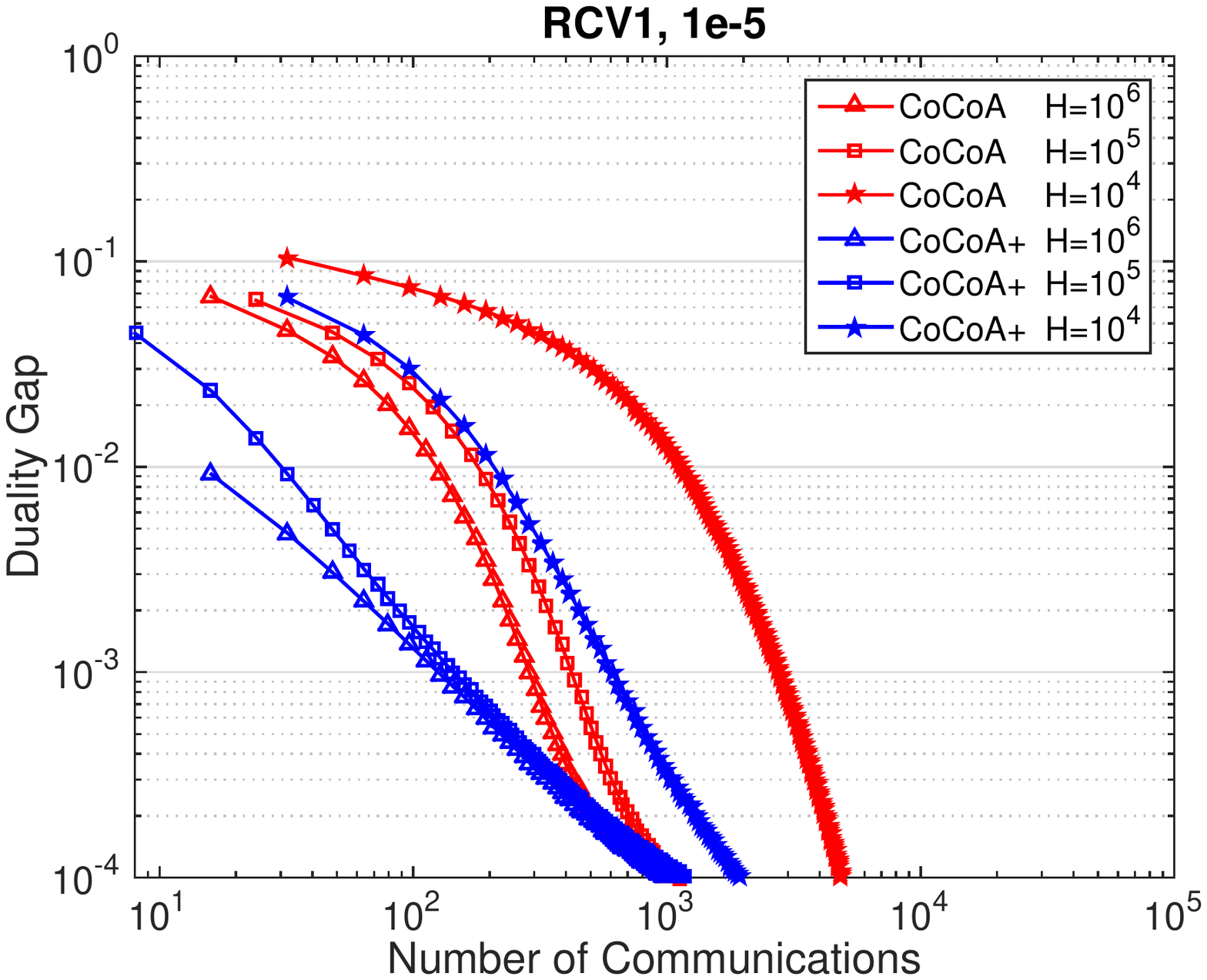}
\smalltrimfig{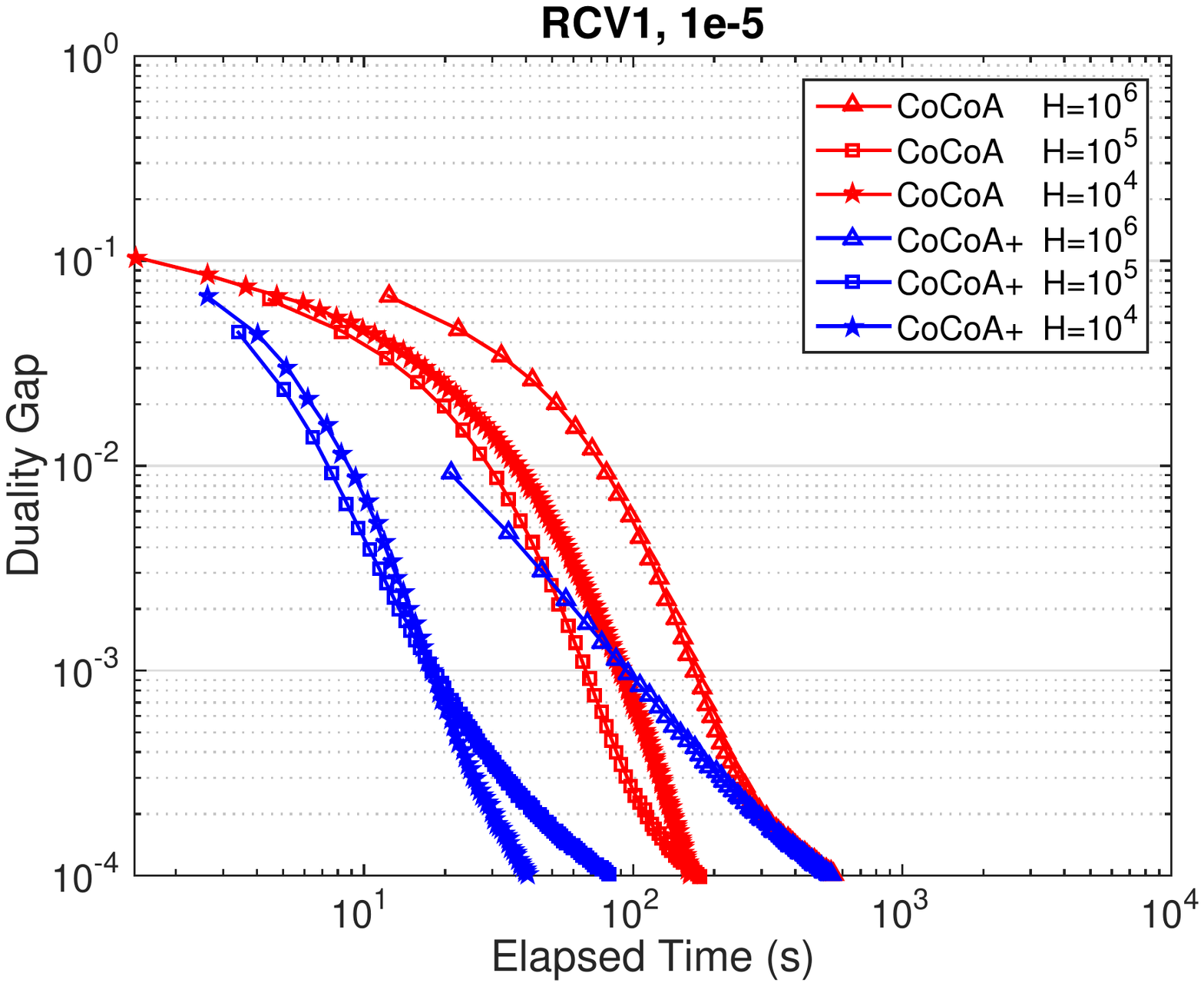} %
\smalltrimfig{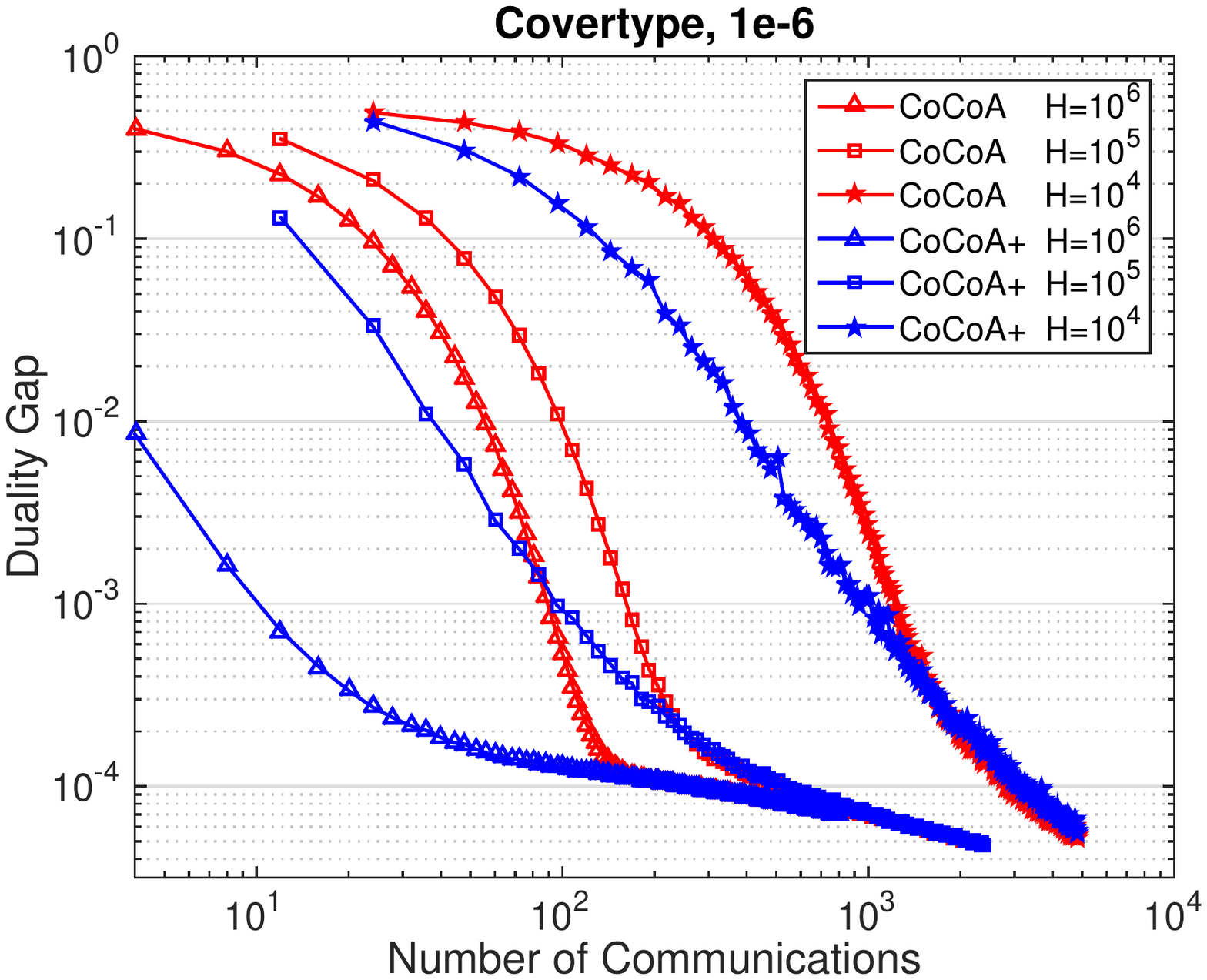}
\smalltrimfig{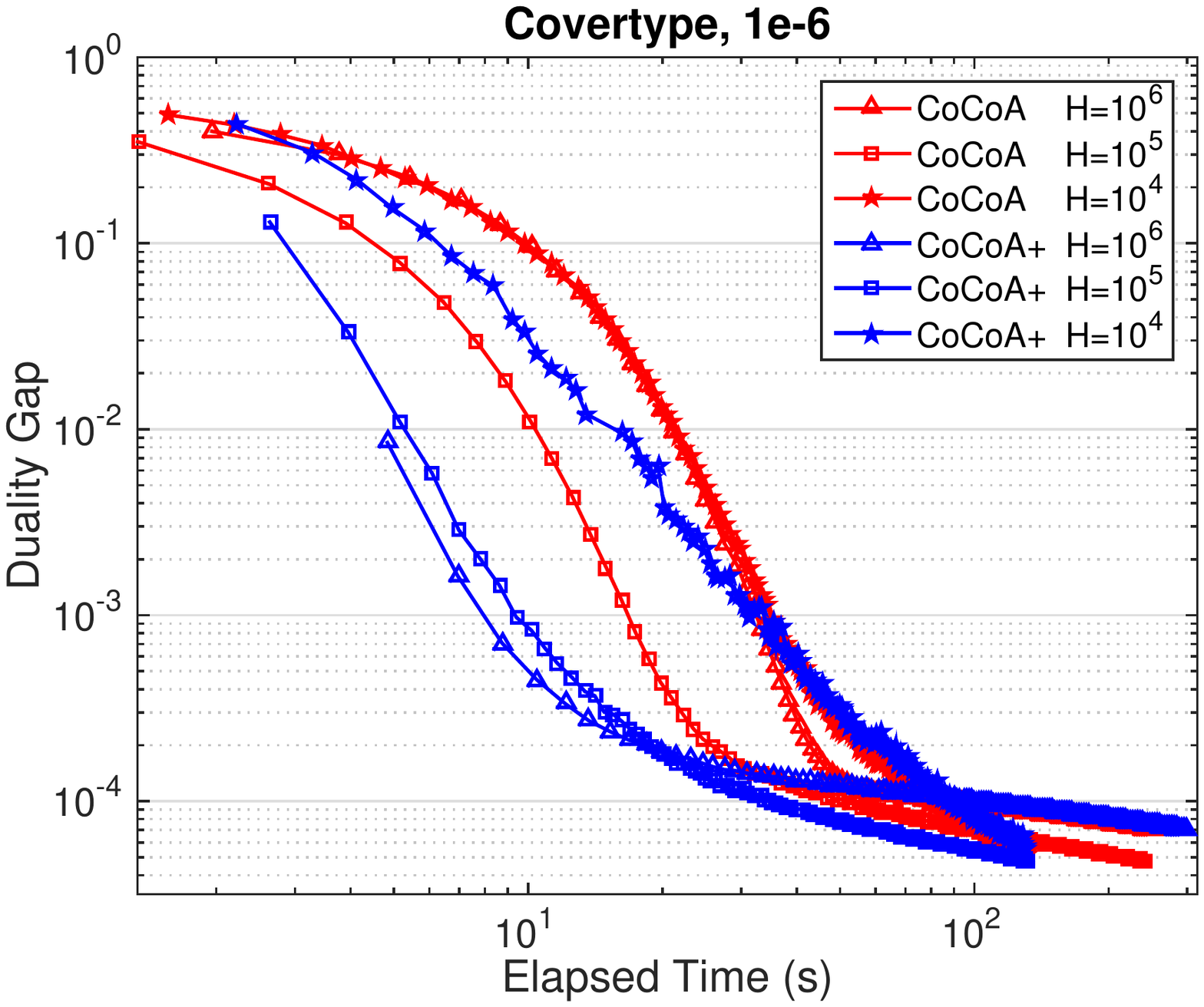}
\smalltrimfig{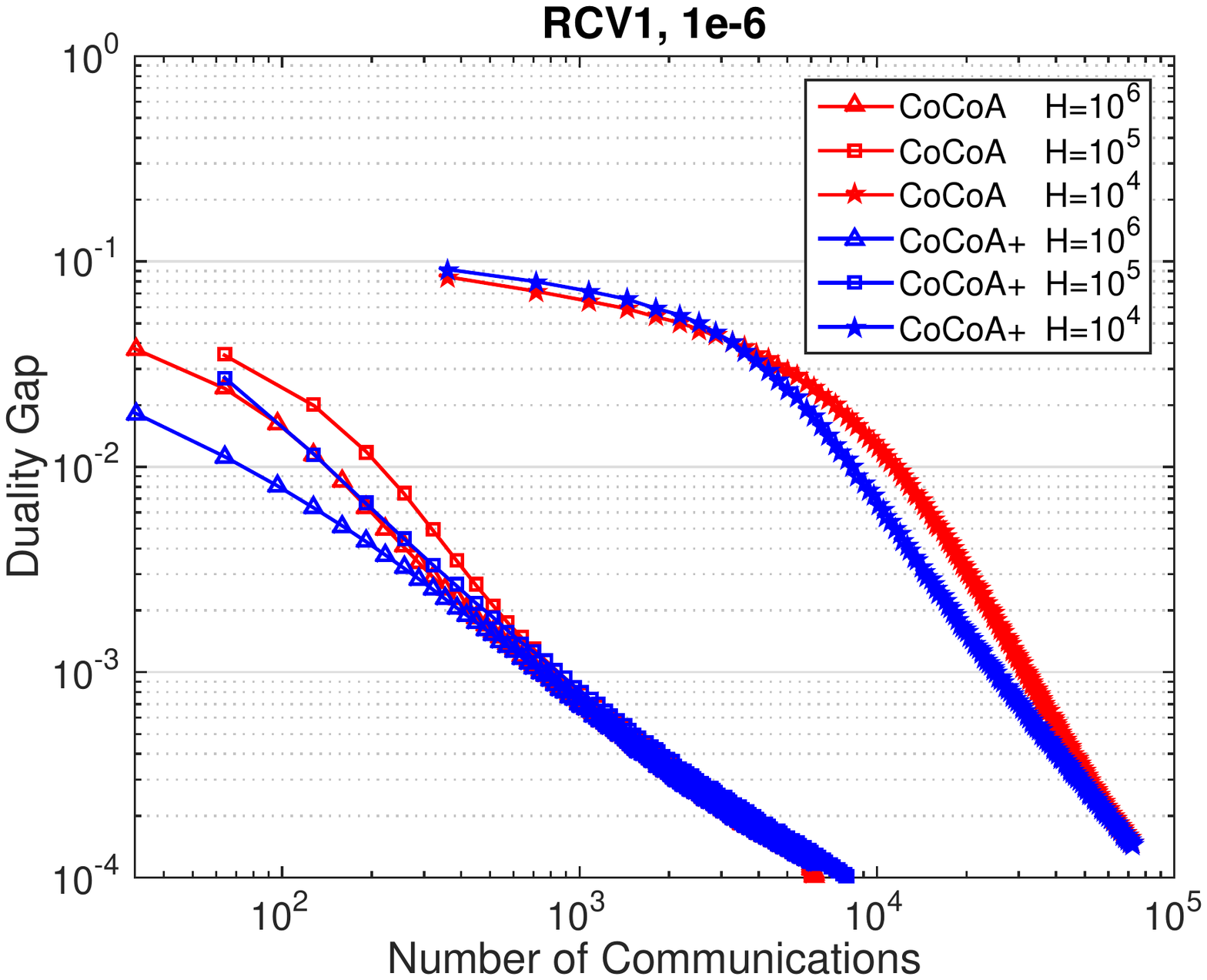}
\smalltrimfig{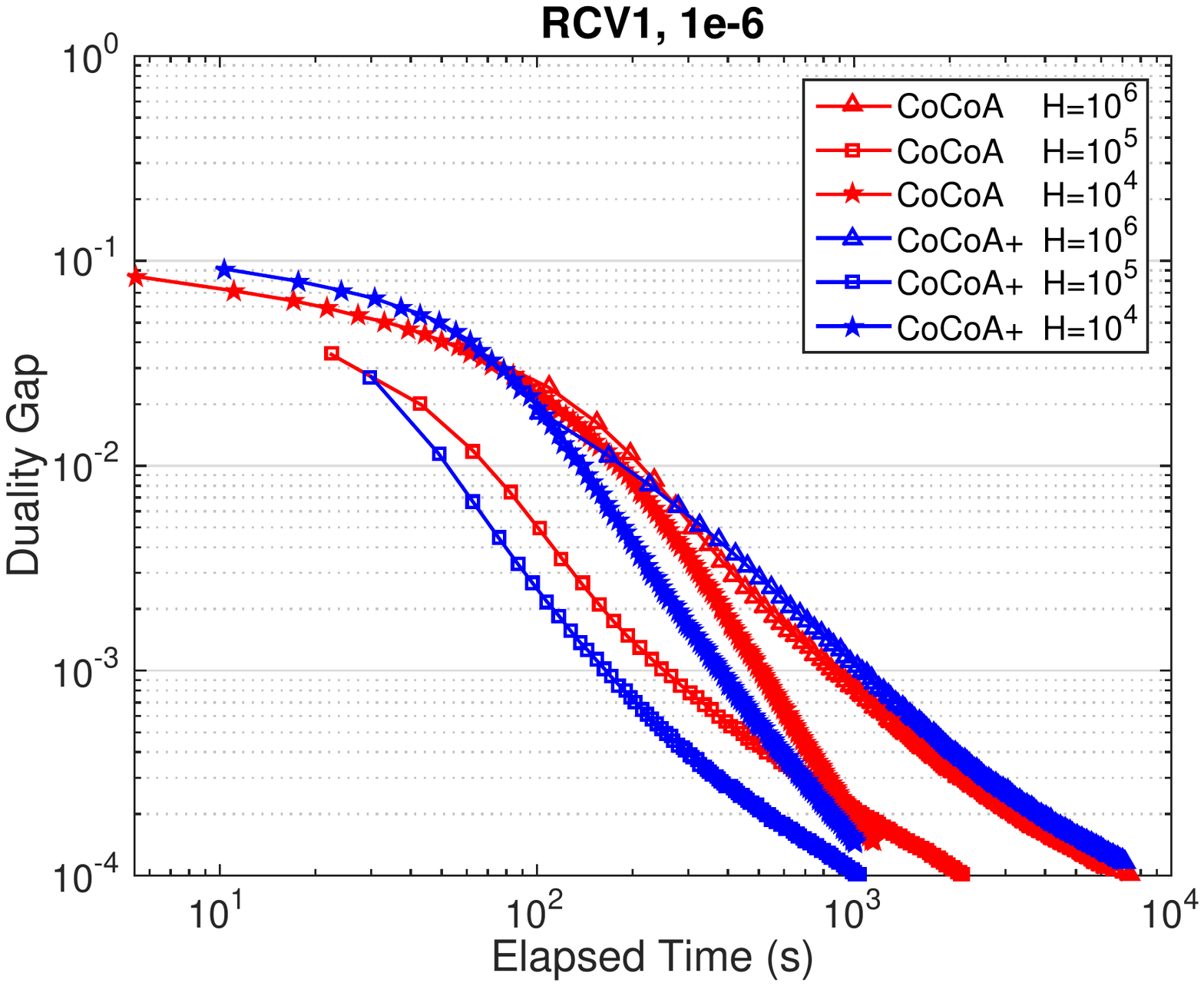}
\vspace{-1.9em}
\caption{Duality gap vs. the number of communicated vectors, as well as duality gap vs. elapsed time in seconds for two datasets: Covertype (left, $K$=4) and RCV1 (right, $K$=8). Both are shown on a log-log scale, and for three different values of regularization ($\lambda$=1e-4; 1e-5; 1e-6). Each plot contains a comparison of \cocoa (red) and \cocoap (blue), for three different values of $H$, the number of local iterations performed per round. For all plots, across all values of $\lambda$ and $H$, we see that \cocoap converges to the optimal solution faster than \cocoa, in terms of both the number of communications and the elapsed time.
\vspace{-.5mm}}
\label{fig:add_avg}
\end{figure*}

\vspace{-1mm}
\paragraph{Methods Allowing Local Optimization.}
Developing methods that allow for local optimization requires carefully devising data-local subproblems to be solved after each communication round. \cite{DANE,DISCO} have proposed distributed Newton-type algorithms in this spirit. However, the subproblems must be solved to high accuracy for convergence to hold, which is often prohibitive as the size of the data on one machine is still relatively large.
In contrast, the \cocoa framework \cite{jaggi2014communication} allows using any local solver of weak local approximation quality in each round. %
By making use of the primal-dual structure in the line of work of \cite{Yu:2012fp,Pechyony:2011wi,Yang:2013vl,Lee:2015vr}, the \cocoa and \cocoap frameworks also allow more control over the aggregation of updates between machines. 
The practical variant DisDCA-p proposed in \cite{Yang:2013vl} allows additive updates but is restricted to SDCA updates, and was proposed without convergence guarantees. 
DisDCA-p can be recovered as a special case of the \cocoap framework when using SDCA as a local solver, if $n_k = n/K$ and $\sigma':=K$, see Appendix~\ref{app:disDCA}. 
The theory presented here also therefore covers that method.

\vspace{-.5em}
\paragraph{ADMM.}
An alternative approach to distributed optimization is to use the alternating direction method of multipliers (ADMM), as used for distributed SVM training in, e.g., \cite{Forero:2010vv}. This uses a penalty parameter balancing between the equality constraint $\wv$ and the optimization objective \cite{boyd2011distributed}. However, the known convergence rates for ADMM are weaker than the more problem-tailored methods mentioned previously, and the choice of the  penalty parameter is often unclear.

\vspace{-.5em}
\paragraph{Batch Proximal Methods.}
In spirit, for the special case of adding ($\aggpar=1$), \cocoap resembles a batch proximal method, using the separable approximation \eqref{eq:subproblem} instead of the original dual \eqref{eq:dual}. Known batch proximal methods require high accuracy subproblem solutions, and don't allow arbitrary solvers of weak accuracy $\Theta$ such as we do here.

\section{Numerical Experiments}
\label{sec:experiments}

We present experiments on several large real-world distributed datasets. 
We show that $\cocoap$ converges faster 
in terms of total rounds as well as elapsed time as compared to \cocoa in all cases, 
despite varying: the dataset, values of regularization, batch size, and cluster size 
(Section \ref{sec:addavg}). In Section \ref{sec:scaling} we demonstrate that this 
performance translates to orders of magnitude improvement in convergence when 
scaling up the number of machines $K$, as compared to \cocoa as well as to several 
other state-of-the-art methods. Finally, in Section~\ref{sec:sigma} we investigate the 
impact of the local subproblem parameter $\sigma'$ in the \cocoap framework.

\vspace{-.8em}
\begin{table}[h]
\caption{Datasets for Numerical Experiments. \vspace{1mm}}
\label{tab:datasets}
   \begin{center}
      \begin{tabular}{l| r | r | r  }
    {\small\textbf{Dataset}} & $n$ &
    $d$ & {\small\textbf{Sparsity}} \\
    \hline
	covertype & 522,911 & 
	  54 & 22.22\% \\
		epsilon & 400,000 &
	  2,000 & 100\% \\
	  RCV1 & 677,399 &
	  47,236 & 0.16\%
      \end{tabular}
   \end{center}\vspace{-1.6em}
\end{table}

\subsection{Implementation Details}
We implement all algorithms in Apache
\textsf{\small Spark}  \cite{Zaharia:2012ve} and run them on m3.large Amazon EC2 instances, applying each method to the binary hinge-loss support vector machine. 
The analysis for this non-smooth loss was not covered in 
\cite{jaggi2014communication} but has been captured here, and thus is both 
theoretically and practically justified. 
The used datasets are summarized in Table \ref{tab:datasets}.

For illustration and ease of comparison, we here use SDCA \cite{ShalevShwartz:2013wl} as the local solver for both \cocoa and \cocoap.
Note that in this special case, and if additionally $\sigma':=K$, and if the partitioning $n_k = n/K$ is balanced, once can show that the \cocoap framework reduces to the practical variant of DisDCA \cite{Yang:2013vl} (which had no convergence guarantees so far).
We include more details on the connection in Appendix~\ref{app:disDCA}.
\vspace{-.5em}

\subsection{Comparison of \cocoap and \cocoa}
\label{sec:addavg}
We compare the \cocoap and \cocoa frameworks directly using two datasets 
(Covertype and RCV1) across various values of $\lambda$, the regularizer, in Figure 
\ref{fig:add_avg}. For each value of $\lambda$ we consider both methods with 
different values of $H$, the number of local iterations performed before 
communicating to the master. For all runs of \cocoap we use the safe upper bound of 
$\aggpar K$ for $\sigma'$. In terms of both the total number of communications 
made and the elapsed time, \cocoap (shown in blue) converges to the optimal solution 
faster than \cocoa (red). The discrepancy is larger for greater values of $\lambda$, 
where the  strongly convex regularizer has more of an impact and the problem 
difficulty is reduced. We also see a greater performance gap for smaller values of $H$, 
where there is frequent communication between the machines and the master, and changes between the algorithms therefore play a larger role.

\subsection{Scaling the Number of Machines $K$}
\label{sec:scaling}

In Figure \ref{fig:scaling_k} we demonstrate the ability of \cocoap to scale with an 
increasing number of machines $K$. The experiments confirm the ability of strong 
scaling of the new method, as predicted by our theory in Section~\ref{sec:convergence}, 
in contrast to the competing methods.
Unlike \cocoa, which becomes linearly slower when increasing the number of 
machines, the performance of \cocoap improves with additional 
machines, only starting to degrade slightly once~$K$=16 for the RCV1 dataset.

\newcommand{\halftrimfig}[1]{\subfigure{\includegraphics[trim = 40 190 40 180, clip, width=.49\linewidth]{#1}}}

\newcommand{\trimfig}[1]{\subfigure{\includegraphics[trim = 25 240 30 240, clip, width=.6\linewidth]{#1}}}
\begin{figure}[ht!]
\centering
\halftrimfig{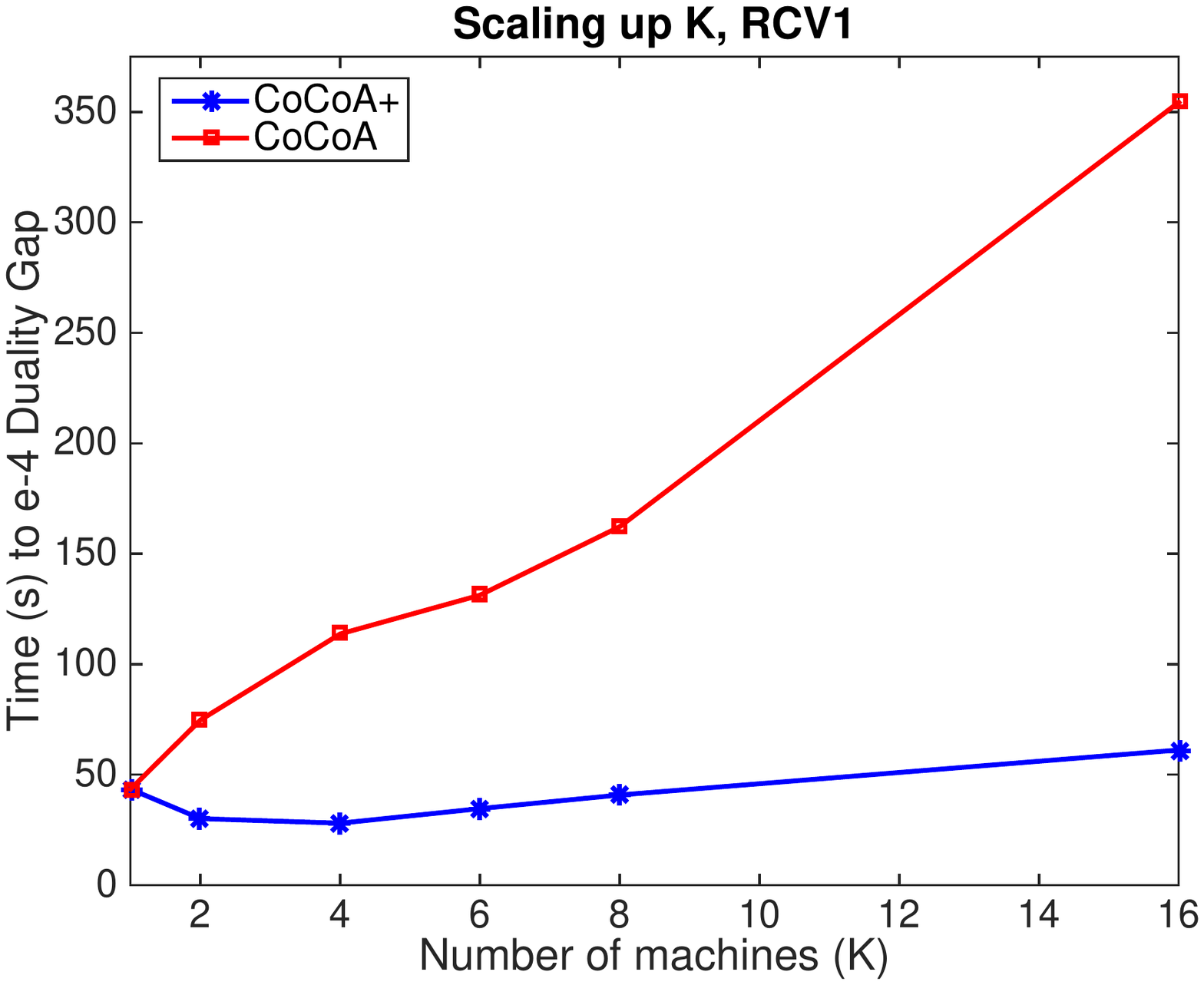}
\halftrimfig{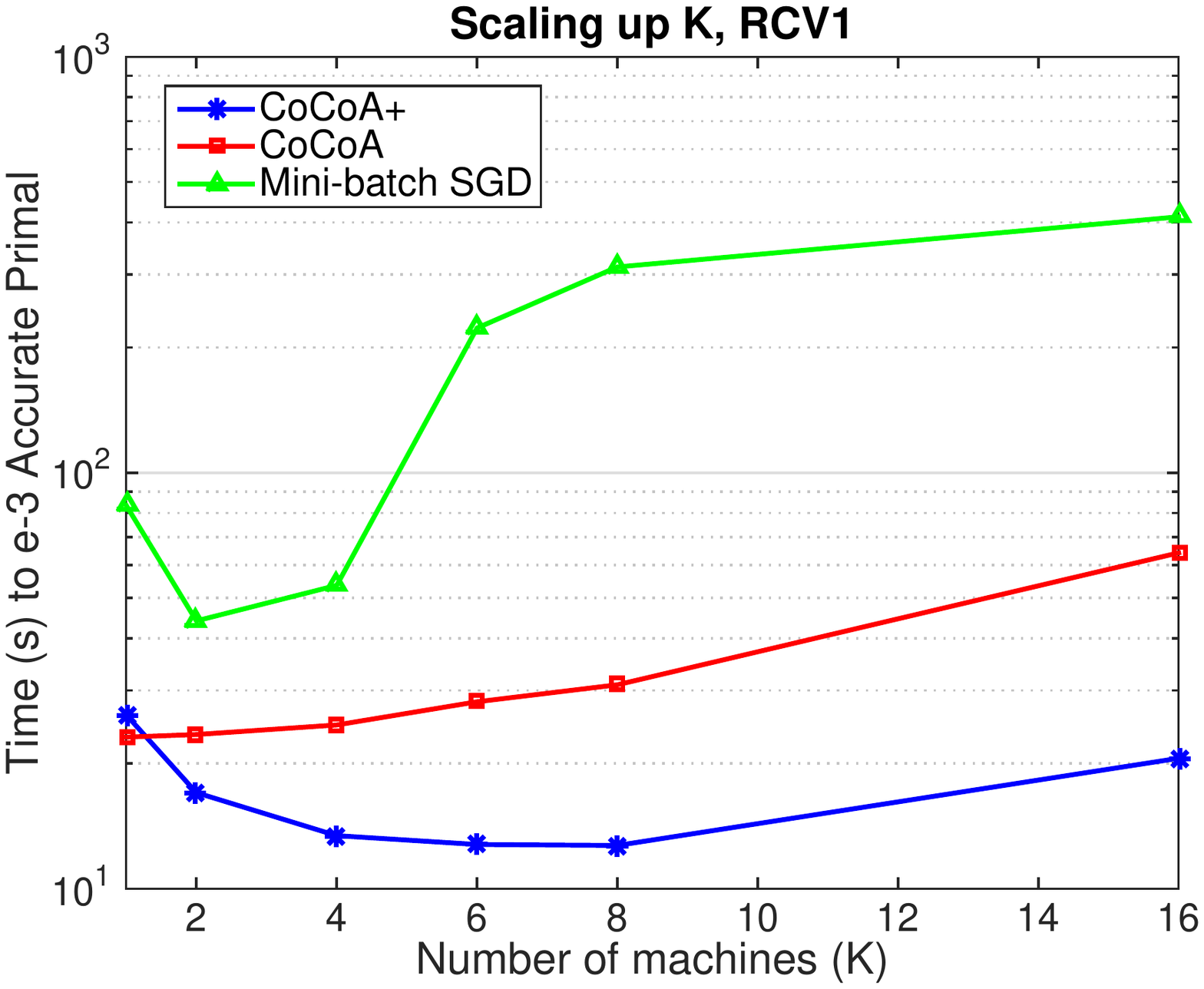}
\trimfig{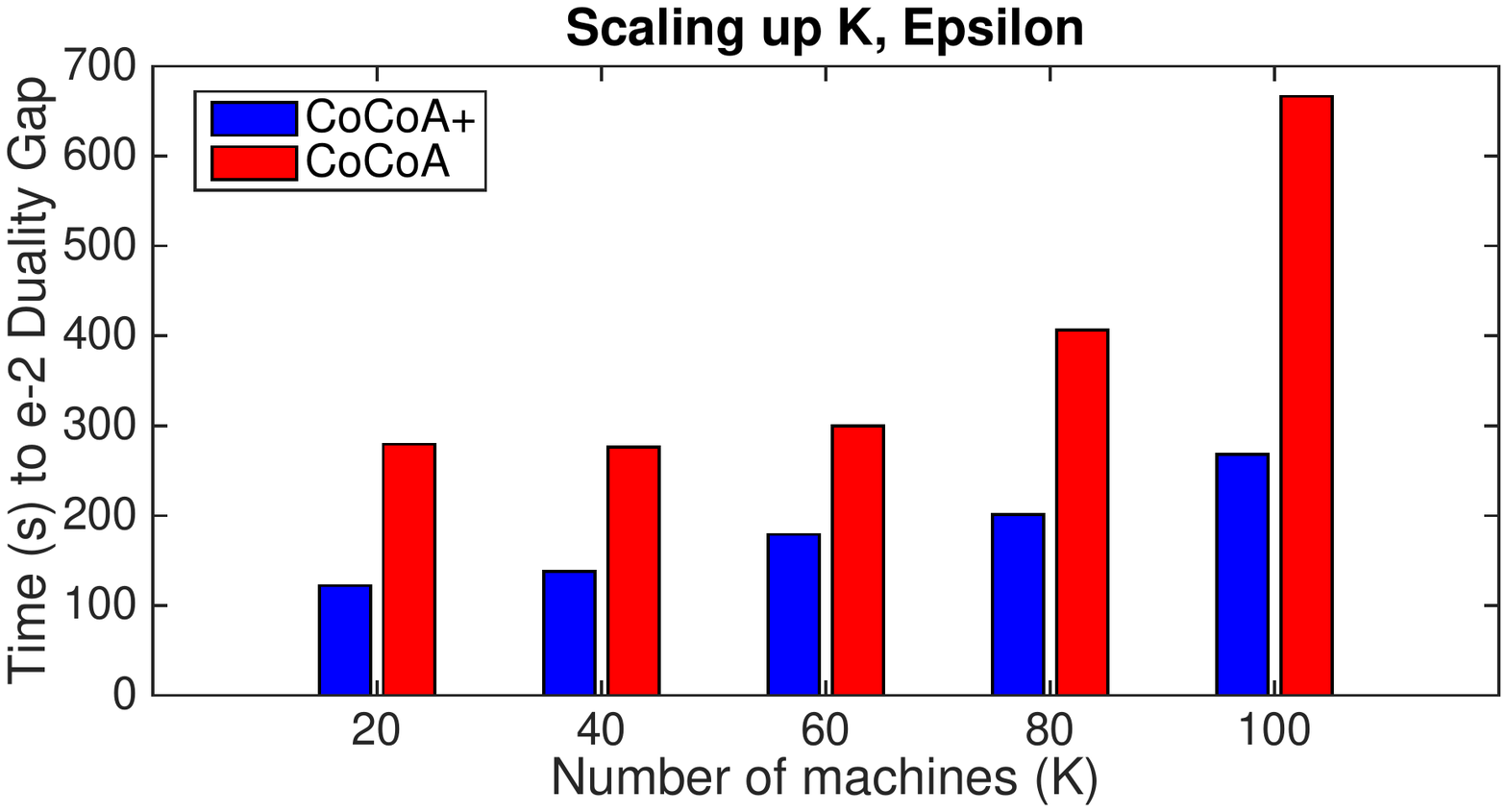}
\vspace{-1em}
\caption{The effect of increasing $K$ on the time (s) to reach an $\epsilon_\bD$-accurate solution. We see that \cocoap converges twice as fast as \cocoa on 100 machines for the Epsilon dataset, and nearly 7 times as quickly for the RCV1 dataset. Mini-batch SGD converges an order of magnitude more slowly than both methods.} \vspace{-1em}
\label{fig:scaling_k}
\end{figure}

\subsection{Impact of the Subproblem Parameter $\sigma'$}
\label{sec:sigma}
Finally, in Figure \ref{fig:sigma}, we consider the effect of the choice of the subproblem parameter $\sigma'$ on convergence. We plot both the number of communications and clock time  on a log-log scale for the RCV1 dataset with $K$=8 and $H$=$1e4$. For $\aggpar=1$ (the most aggressive variant of \cocoap in which updates are added) we consider several different values of~$\sigma'$, ranging from $1$ to $8$. The value $\sigma'$=8 represents the safe upper bound of $\aggpar K$. The optimal convergence occurs around $\sigma'$=4, and diverges for $\sigma' \le 2$.
Notably, we see that the easy to calculate upper bound of $\sigma':=\aggpar K$ %
(as given by Lemma \ref{lem:sigmaPrimeNotBad})
has only slightly worse performance than best possible subproblem parameter in our setting. %
\vspace{-1em}

\begin{figure}[h!]
\halftrimfig{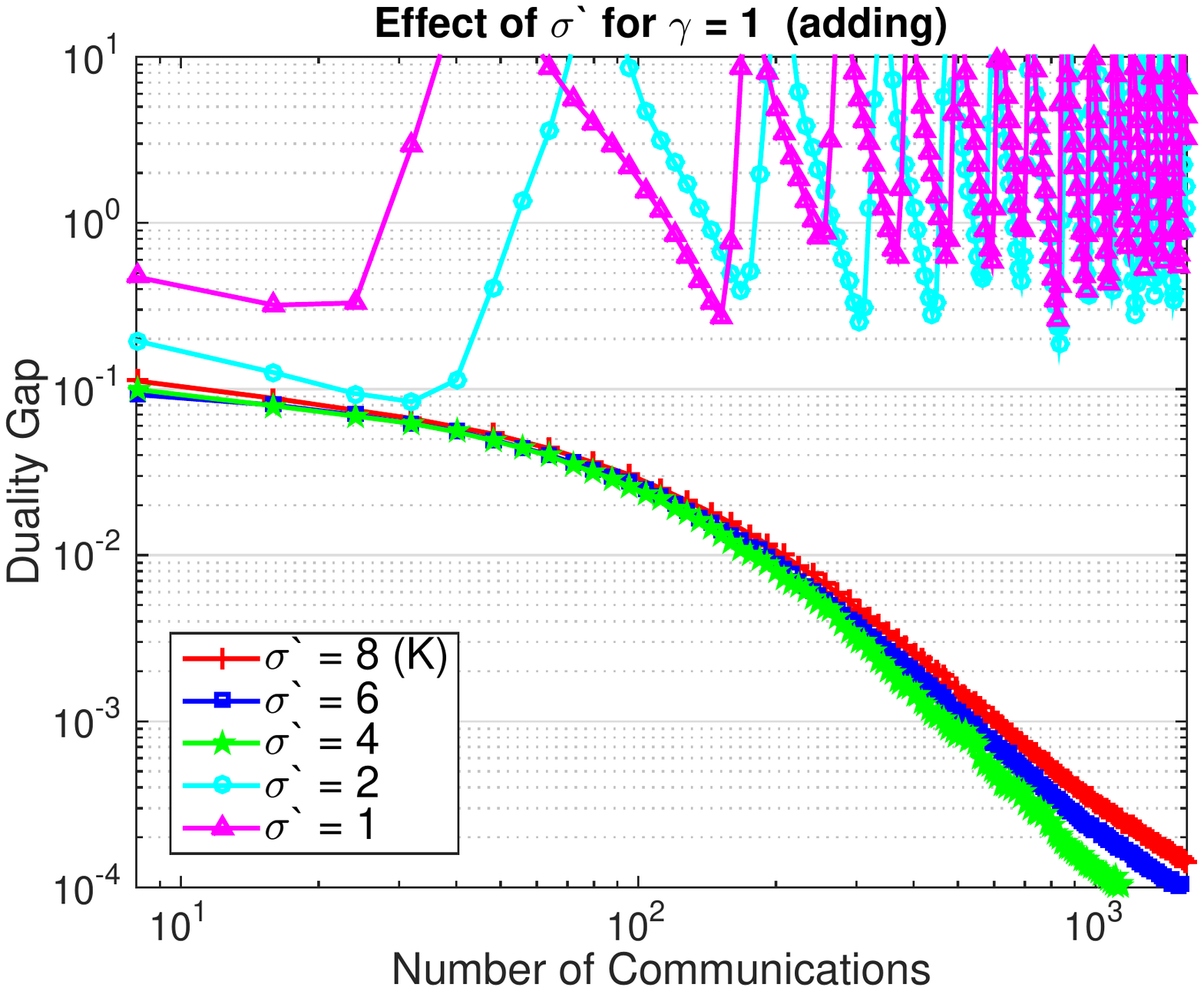}
\halftrimfig{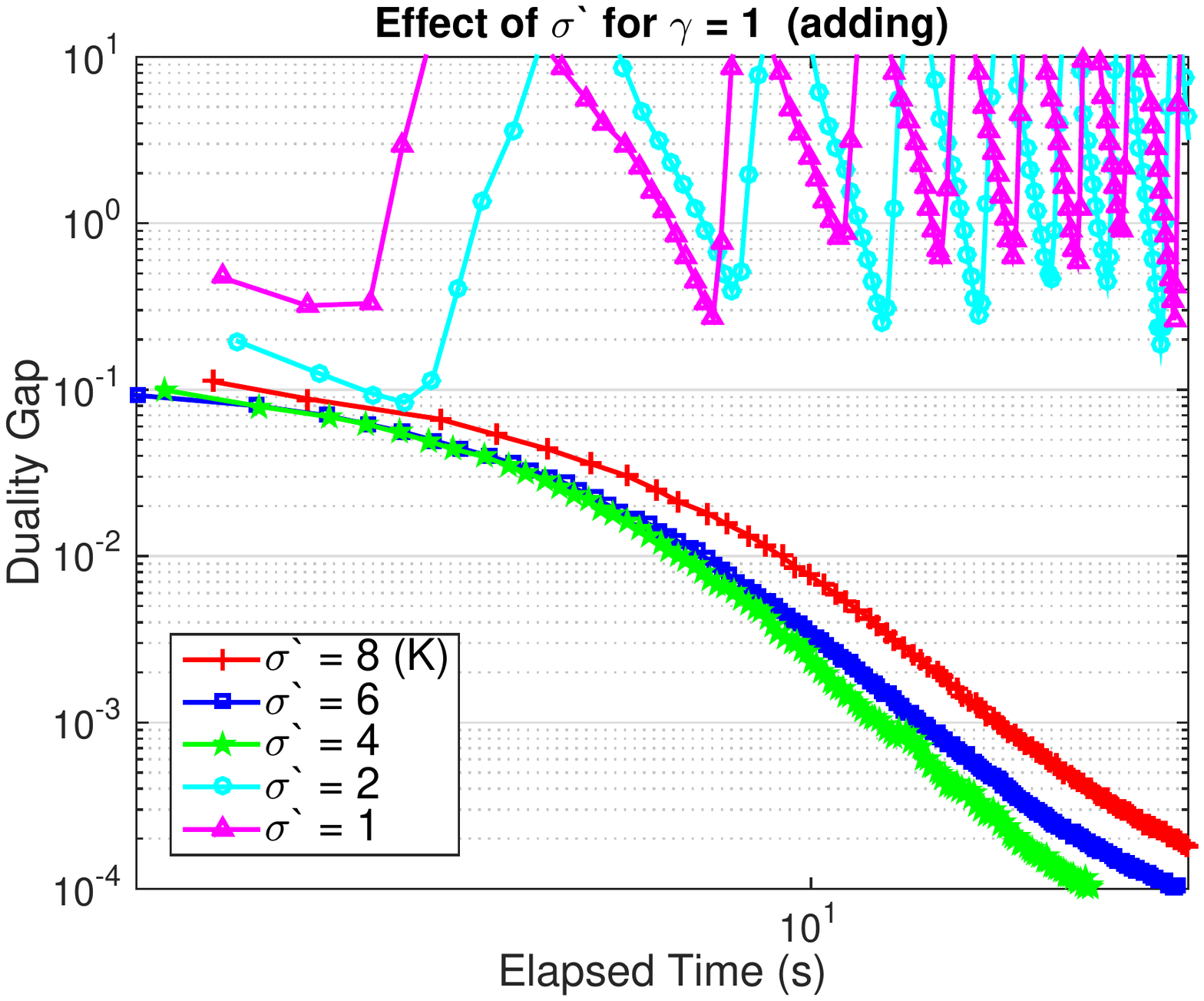}
\vspace{-1em}
\caption{The effect of $\sigma'$ on convergence of $\cocoap$ for the RCV1 dataset distributed across $K$=8 machines. Decreasing $\sigma'$ improves performance in terms of communication and overall run time until a certain point, after which the algorithm diverges. The ``safe'' upper bound of $\sigma'$:=$K$=8 has only slightly worse performance than the practically best ``un-safe'' value of $\sigma'$.}\vspace{-1em}
\label{fig:sigma}
\end{figure}

\section{Conclusion}
In conclusion, we present a novel framework \cocoap that allows for fast and  
communication-efficient \textit{additive aggregation} in distributed 
algorithms for primal-dual optimization. 
We analyze the theoretical performance of this method, giving strong 
primal-dual convergence rates with outer iterations scaling independently of 
the number of machines. 
We extended our theory to allow for non-smooth losses. Our 
experimental results show significant speedups over previous methods, including the 
original \cocoa framework as well as other state-of-the-art methods.

\vspace{-.4em}
\paragraph{Acknowledgments.}
We thank Ching-pei Lee and an anonymous reviewer for several helpful insights and comments.

\bibliography{minibatch}
\bibliographystyle{icml2015}

\clearpage
\appendix
 \onecolumn
\part*{Appendix}

\section{Technical Lemmas}

\begin{lemma}
[Lemma 21 in \cite{ShalevShwartz:2013wl}]
\label{lemma:ajvoiewffa}
Let $\ell_i : \R \to \R$ be an 
$L$-Lipschitz continuous. Then for any real value $a$ with $|a|> L$ we have that
$\ell_i^*(a) = \infty$.
\end{lemma}

\begin{lemma}
\label{lemma:asfewfawfcda}
Assuming the loss functions $\ell_i$ are bounded by $\ell_i(0) \leq 1$ for all $i\in[n]$ (as we have assumed in \eqref{eq:afswfevfwaefa} above), then 
for the zero vector $\vc{\alphav}{0}
 := {\bf 0}\in \R^n$, we have
\begin{equation}
\label{eq:afjfjaoefvcwa}
\bD(\alphav^*)
 - \bD(\vc{\alphav}{0})
= 
\bD(\alphav^*)
-\bD({\bf 0})
 \leq 1.
 \end{equation}
\end{lemma}
\begin{proof}
For $\alphav := {\bf 0}\in \R^n$, we have
$\wv(\alphav) = 
\frac1{\lambda n}
A \alphav 
 = {\bf 0} \in \R^d$.
 Therefore, by definition of the dual objective $\bD$ given in~\eqref{eq:dual},
\begin{align*}
0 &\leq \bD(\alphav^*)
-\bD(\alphav)
\leq \bP(\wv(\alphav)) - \bD(\alphav)
 = 0 - \bD(\alphav)
 \overset{\eqref{eq:afswfevfwaefa},\eqref{eq:dual}
}{\leq} 1. \qedhere
\end{align*} 
\end{proof}

\section{Proofs}

\subsection{Proof of Lemma \ref{lem:RelationOfDTOSubproblems}}
  
Indeed, we have
\begin{align}
\label{eq:afijwfcewa}
\bD(\alphav
+\aggpar 
\sum_{k=1}^K
\vsubset{\Delta \alphav}{k}
)
&=
\underbrace{-\frac1n\sum_{i=1}^n
\ell_i^*(-\alpha_i
    -\aggpar (\sum_{k=1}^K
     \vsubset{\Delta \alphav}{k})_i)}_{A} -\frac\lambda2 
\underbrace{\Big\| \frac1{\lambda n}
A (\alphav + \aggpar 
  \sum_{k=1}^K \vsubset{\Delta \alphav}{k}) \Big\|^2}_B.  
\end{align}
Now, let us bound the terms $A$ and $B$ separately.
We have
\begin{align*}
A
&=
 -\frac1n\sum_{k=1}^K
 \left(
 \sum_{i\in \mathcal{P}_k}
 \ell_i^*(-\alpha_i-\aggpar 
  (\vsubset{\Delta \alphav}{k})_i)
 \right)
=
 -\frac1n\sum_{k=1}^K
 \left(
 \sum_{i\in \mathcal{P}_k}
 \ell_i^*(-(1-\aggpar)
   \alpha_i-\aggpar 
  (\alphav + \vsubset{\Delta \alphav}{k})_i)
 \right) 
\\
&\geq 
 -\frac1n\sum_{k=1}^K
 \left(
 \sum_{i\in \mathcal{P}_k}
 (1-\aggpar) \ell_i^*(-
   \alpha_i)
   +\aggpar 
  \ell_i^*(-(\alphav + \vsubset{\Delta \alphav}{k})_i)
 \right). 
\end{align*}
Where the last inequality is due to Jensen's inequality.
Now we will bound $B$, using the safe separability measurement $\sigma'$ as defined in \eqref{eq:sigmaPrimeSafeDefinition}.\vspace{-1mm}
\begin{align*}
B
&=
\Big\|\frac1{\lambda n}
A (\alphav + \aggpar 
  \sum_{k=1}^K \vsubset{\Delta \alphav}{k}) \Big\|^2
=
\Big\|\wv(\alphav) + \aggpar\frac1{\lambda n}
   \sum_{k=1}^K A\vsubset{\Delta \alphav}{k}  \Big\|^2  
\\
&   =
\|\wv(\alphav)   \|^2
+
\sum_{k=1}^K
2\aggpar\frac1{\lambda n} 
\wv(\alphav)^T  
    A\vsubset{\Delta \alphav}{k}  
+
\aggpar
\Big(\frac1{\lambda n}\Big)^2
\aggpar
\Big\|  
   \sum_{k=1}^K A\vsubset{\Delta \alphav}{k}  \Big\|^2  
\\
& \overset{\eqref{eq:sigmaPrimeSafeDefinition}}
{\leq}
\|\wv(\alphav)   \|^2
+
\sum_{k=1}^K
2\aggpar\frac1{\lambda n} 
\wv(\alphav)^T  
    A\vsubset{\Delta \alphav}{k}  
   +
\aggpar
\Big(\frac1{\lambda n}\Big)^2
\sigma'
   \sum_{k=1}^K \|A \Delta 
     \vsubset{\alphav}{k}\|^2.    
\end{align*}  
Plugging $A$ and $B$ into 
  \eqref{eq:afijwfcewa} 
  will give us
\begin{align*}
\nonumber
 \bD(\alphav
+\aggpar 
\sum_{k=1}^K
\vsubset{\Delta \alphav}{k}
)
 \ge&
-\frac1n\sum_{k=1}^K
 \left(
 \sum_{i\in \mathcal{P}_k}
 (1-\aggpar) \ell_i^*(-
   \alpha_i)
   +\aggpar 
  \ell_i^*(-(\alphav + \vsubset{\Delta \alphav}{k})_i)
 \right)
\\
&   
-\aggpar  \frac\lambda2 \|\wv(\alphav)   \|^2
-(1-\aggpar)  \frac\lambda2 \|\wv(\alphav)   \|^2
-\frac\lambda2  
\sum_{k=1}^K
2\aggpar\frac1{\lambda n} 
\wv(\alphav)^T  
    A\vsubset{\Delta \alphav}{k}  
 -\frac\lambda2 
\aggpar
\Big(\frac1{\lambda n}\Big)^2
\sigma'
   \sum_{k=1}^K \|A \Delta \vsubset{\alphav}{k}\|^2
\\%
 =&
 \underbrace{
 -\frac1n\sum_{k=1}^K
 \left(
 \sum_{i\in \mathcal{P}_k}
 (1-\aggpar) \ell_i^*(-
   \alpha_i) 
 \right)
-(1-\aggpar)  \frac\lambda2 \|\wv(\alphav)   \|^2 
}_{(1-\aggpar) \bD(\alphav)}
\\
& 
+  
\aggpar 
 \sum_{k=1}^K
 \left(
 -\frac1n
 \sum_{i\in \mathcal{P}_k}
   \ell_i^*(-(\alphav + \vsubset{\Delta \alphav}{k})_i)
   - \frac1{K} \frac\lambda2 \|\wv(\alphav)   \|^2
   -   
  \frac1{n} 
\wv(\alphav)^T  
    A\vsubset{\Delta \alphav}{k}
    -\frac\lambda2  
\sigma'  \Big\|\frac1{\lambda n}A \Delta \vsubset{\alphav}{k} \Big\|^2    
 \right)
\\
\overset{\eqref{eq:subproblem}}{=}& (1-\aggpar) \bD(\alphav)
 +\aggpar \sum_{k=1}^K \Ggk(  \vsubset{\Delta \alphav}{k}; \wv, \vsubset{\alphav}{k}).   
\end{align*}

\subsection{Proof of Lemma \ref{lem:sigmaPrimeNotBad}}

See \cite{richtarik2013distributed}. %

\subsection{Proof of Lemma \ref{lem:basicLemma}}
For sake of notation, 
we will write 
$\alphav$ instead of $\vc{\alphav}{t}$,
$\wv$ instead of $\wv(\vc{\alphav}{t})$
and
$\uv$ instead of $\vc{\uv}{t}$.

Now, let us estimate the expected change of the dual objective. 
Using the definition of the dual update $\vc{\alphav}{t+1} := \vc{\alphav}{t} + \aggpar \, \sum_k \vsubset{\Delta \alphav}{k}$ resulting in Algorithm~\ref{alg:cocoa}, we have
\begin{align*}
\Exp\big[\bD(\vc{\alphav}{t})
 - \bD(\vc{\alphav}{t+1})\big]
& =
\Exp\Big[\bD(\alphav)
 - \bD(\alphav +
  \aggpar \sum_{k=1}^K
  \vsubset{\Delta \alphav}{k})\Big]
\\
& \text{(by Lemma \ref{lem:RelationOfDTOSubproblems} on the local function $\Ggk(\alphav;\wv, \vsubset{\alphav}{k})$ approximating the global objective $\bD(\alphav)$)}\\
&\leq
\Exp\Big[\bD(\alphav)
-(1-\aggpar)\bD(\alphav)
-\aggpar 
 \sum_{k=1}^K 
 \Ggk (\vsubset{
 \vc{\Delta \alphav}{t}}{k}; \wv, \vsubset{\alphav}{k})
\Big]\\
&=
\aggpar
\Exp\Big[
 \bD(\alphav)
- 
 \sum_{k=1}^K 
 \Ggk (\vsubset{
 \vc{\Delta \alphav}{t}}{k}; \wv, \vsubset{\alphav}{k})
\Big]
\\
&
=
\aggpar
\Exp\Big[
 \bD(\alphav)
 -
 \sum_{k=1}^K 
 \Ggk(\vsubset{\Delta \alphav^*}{k};\wv, \vsubset{\alphav}{k})
 +
 \sum_{k=1}^K 
 \Ggk(\vsubset{\Delta \alphav^*}{k};\wv, \vsubset{\alphav}{k})
- 
 \sum_{k=1}^K 
 \Ggk (\vsubset{
 \vc{\Delta \alphav}{t}}{k}; \wv, \vsubset{\alphav}{k})
\Big]
\\
&\text{(by the notion of quality \eqref{eq:localSolutionQuality} of the local solver, as in Assumption \ref{asm:THeta})}\\
&\leq
\aggpar
\bigg(
 \bD(\alphav)
 -
 \sum_{k=1}^K 
 \Ggk(\vsubset{\Delta \alphav^*}{k};\wv, \vsubset{\alphav}{k})
 +
 \Theta
 \Big(
 \sum_{k=1}^K  
 \Ggk(\vsubset{\Delta \alphav^*}{k};\wv, \vsubset{\alphav}{k})
 -
\underbrace{  \sum_{k=1}^K  
 \Ggk({\bf 0};\wv, \vsubset{\alphav}{k})
 }_{\bD(\alphav)}
 \Big)
\bigg)
\\
&=
\aggpar
(1-\Theta)
\Big(
\underbrace{
 \bD(\alphav)
 -
 \sum_{k=1}^K 
 \Ggk(\vsubset{\Delta \alphav^*}{k};\wv, \vsubset{\alphav}{k})
 }_{C}
\Big).
\tagthis
\label{eq:Afasfwafewaef}
\end{align*} 
Now, let us upper bound 
the $C$ term 
(we will denote by
$\Delta \alphav^* 
 = \sum_{k=1}^K \vsubset{\Delta \alphav^*}{k}$):
\begin{align*}
C&
\overset{\eqref{eq:dual},
\eqref{eq:subproblem}}{=}
   \frac1n 
 \sum_{i =1}^n 
 \left(
\ell_i^*(-\alpha_i - \Delta \alphav^*_i)
-\ell_i^*(- \alpha_i)
\right)
 +\frac1n  
\wv(\alphav)^T A  \Delta \alphav^*
 + \sum_{k=1}^K 
\frac\lambda2
 \sigma'   \Big\|\frac1{\lambda n} A \vsubset{\Delta \alphav^*}{k}\Big\|^2
\\
&\leq  
   \frac1n 
 \sum_{i =1}^n 
 \left(
\ell_i^*(-\alpha_i - s (u_i - \alpha_i))
-\ell_i^*(- \alpha_i)
\right)
 +\frac1n  
\wv(\alphav)^T A  s (\uv  - \alphav )
 + \sum_{k=1}^K 
\frac\lambda2
 \sigma'   \Big\|\frac1{\lambda n} A \vsubset{s (\uv  - \alphav )}{k}\Big\|^2
\\
&\overset{\mbox{Strong conv.}}{\leq} 
   \frac1n 
 \sum_{i =1}^n 
 \left(
s \ell_i^*(-u_i )
+
(1-s)
\ell_i^*(-\alpha_i )
-
\frac{\mu}{2}
(1-s)s (u_i -\alpha_i)^2
-\ell_i^*(- \alpha_i)
\right)
 +\frac1n  
\wv(\alphav)^T A  s (\uv  - \alphav )
\\& \quad\quad\quad\quad\quad + \sum_{k=1}^K 
\frac\lambda2
 \sigma'   \Big\|\frac1{\lambda n} A \vsubset{s (\uv  - \alphav )}{k}\Big\|^2 
\\
&=
   \frac1n 
 \sum_{i =1}^n 
 \left(
s \ell_i^*(-u_i )
  -s 
\ell_i^*(-\alpha_i )
-
\frac{\mu}{2}
(1-s)s (u_i -\alpha_i)^2
\right)
 +\frac1n  
\wv(\alphav)^T A  s (\uv  - \alphav )
  + \sum_{k=1}^K 
\frac\lambda2
 \sigma'   \Big\|\frac1{\lambda n} A \vsubset{s (\uv  - \alphav )}{k}\Big\|^2.  
\end{align*}
The convex conjugate maximal property implies that
\begin{equation}
\label{eq:adjwofcewa}
\ell_i^*(-u_i)
= -u_i \wv(\alphav)^T \xv_i
  -\ell_i(\wv(\alphav)^T \xv_i).
\end{equation}
Moreover, from the definition of the primal and dual optimization problems \eqref{eq:primal},
\eqref{eq:dual}, we can write the duality gap as
\begin{align}
\label{eq:asdfjiwjfeojawfa}
\gap(\alphav) := \bP(\wv(\alphav))-\bD(\alphav)
&\overset{
\eqref{eq:primal},
\eqref{eq:dual}
}{=}
 \frac1{\N} 
 \sum_{i=1}^\N
 \left(
  \ell_i( \xv_j^T \wv) 
 +  \ell_i^*(- \alpha_i)
 + \wv(\alphav)^T \xv_i \alpha_i
 \right).  
\end{align}
Hence,
\begin{align*}
C
&\overset{
\eqref{eq:adjwofcewa}}
{\leq}
  \frac1n 
 \sum_{i =1}^n 
 \left( 
-s u_i \wv(\alphav)^T \xv_i
  -s\ell_i(\wv(\alphav)^T \xv_i)
  -s 
\ell_i^*(-\alpha_i )
\underbrace{-s \wv(\alphav)^T \xv_i \alpha_i
+s \wv(\alphav)^T \xv_i \alpha_i
}_{0}
-
\frac{\mu}{2}
(1-s)s (u_i -\alpha_i)^2
\right)
\\&\qquad  +\frac1n  
\wv(\alphav)^T A  s (\uv  - \alphav )
 + \sum_{k=1}^K 
\frac\lambda2
 \sigma'   \Big\|\frac1{\lambda n} A \vsubset{s (\uv  - \alphav )}{k}\Big\|^2 
\\
&=
  \frac1n 
 \sum_{i =1}^n 
 \left( 
  -s\ell_i(\wv(\alphav)^T \xv_i)
  -s\ell_i^*(-\alpha_i )
  -s \wv(\alphav)^T \xv_i \alpha_i
\right)
+
  \frac1n 
 \sum_{i =1}^n 
 \left(  s \wv(\alphav)^T \xv_i
( \alpha_i-u_i )
 -
\frac{\mu}{2}
(1-s)s (u_i -\alpha_i)^2
\right)
\\&\qquad  +\frac1n  
\wv(\alphav)^T A  s (\uv  - \alphav )
 + \sum_{k=1}^K 
\frac\lambda2
 \sigma'   \Big\|\frac1{\lambda n} A \vsubset{s (\uv  - \alphav )}{k}\Big\|^2  
\\
&\overset{\eqref{eq:asdfjiwjfeojawfa}}{=}
 -s \gap(\alphav)
-
\frac{\mu}{2}
(1-s)s 
  \frac1n 
 \sum_{i =1}^n 
 \|\uv-\alphav\|^2 
 + 
\frac{\sigma'}{2\lambda }
(\frac s{  n})^2
\sum_{k=1}^K   
  \| A \vsubset{  (\uv  - \alphav )}{k}\|^2.
  \tagthis 
  \label{eq:asdfafdas}
 \end{align*}
Now, the claimed improvement bound
\eqref{eq:lemma:dualDecrease_VS_dualityGap}
follows
by plugging 
\eqref{eq:asdfafdas}
into \eqref{eq:Afasfwafewaef}.

\subsection{Proof of Lemma 
\ref{lemma:BoundOnR}}

For general convex functions, the strong convexity parameter is 
$\mu=0$, and hence the definition of $\vc{R}{t}$ becomes
\begin{align*} 
\vc{R}{t}
\overset{\eqref{eq:defOfR}}{=}
  \sum _{k=1}^K   
  \| A \vsubset{  (\vc{\uv} {t} - \vc{\alphav}{t} )}{k}\|^2
\overset{\eqref{eq:definitionOfSigmaK}}{\leq}   
\sum _{k=1}^K 
\sigma_k  
  \|   \vsubset{  (\vc{\uv} {t} - \vc{\alphav}{t} )}{k}\|^2
\overset{\mbox{Lemma \ref{lemma:ajvoiewffa}}}{\leq}   
\sum _{k=1}^K 
\sigma_k  |\mathcal{P}_k| 4L^2.
\end{align*}

\subsection{Proof of Theorem \ref{thm:convergenceNonsmooth}}

At first let us estimate expected change of dual feasibility. By using the main Lemma \ref{lem:basicLemma}, we have
\begin{align*} 
 \Exp[\bD(\alphav^*)-\bD(\vc{\alphav}{t+1})]
 &=
\Exp[\bD(\alphav^*)-\bD(\vc{\alphav}{t+1})+\bD(\vc{\alphav}{t})-\bD(\vc{\alphav}{t})]
\\
&
\overset{\eqref{eq:lemma:dualDecrease_VS_dualityGap}
}{=}
\bD(\alphav^*)-\bD(\vc{\alphav}{t})
-\aggpar
(1-\Theta)  
 s \gap(\vc{\alphav}{t})
+
\aggpar
(1-\Theta)
\tfrac{\sigma'}{2\lambda }
(\frac s{  n})^2
\vc{R}{t}
\\
&
\overset{\eqref{eq:gap}
}{=}
\bD(\alphav^*)-\bD(\vc{\alphav}{t})
-\aggpar
(1-\Theta)
   s  (\bP(\wv(\vc{\alphav}{t}))-\bD(\vc{\alphav}{t}))
+
\aggpar
(1-\Theta)  \tfrac{\sigma'}{2\lambda }
(\frac s{  n})^2
\vc{R}{t} 
\\
&\leq
\bD(\alphav^*)-\bD(\vc{\alphav}{t})
-\aggpar
(1-\Theta)
 s  (\bD(\alphav^* )-\bD(\vc{\alphav}{t}) )
+
\aggpar
(1-\Theta) 
\tfrac{\sigma'}{2\lambda }
(\frac s{  n})^2
\vc{R}{t} \\
&
\overset{\eqref{eq:asfjoewjofa}}{\leq} 
\left( 
 1-\aggpar
(1-\Theta)
   s
\right) 
   (\bD(\alphav^* )-\bD(\vc{\alphav}{t}))
+
\aggpar
(1-\Theta) 
\tfrac{\sigma'}{2\lambda }
(\frac s{  n})^2
4L^2  \sigma.
\tagthis 
\label{eq:asoifejwofa}
\end{align*}
 Using
\eqref{eq:asoifejwofa}
recursively we have 
 \begin{align*} 
 \Exp[\bD(\alphav^*)-\bD(\vc{\alphav}{t})]
 &=
\left( 
 1-\aggpar
(1-\Theta)
   s
\right)^t 
   (\bD(\alphav^* )-\bD(\vc{\alphav}{0}))
+
\aggpar
(1-\Theta) 
\tfrac{\sigma'}{2\lambda }
(\frac s{  n})^2
4L^2  \sigma 
\sum_{j=0}^{t-1}
\left( 
 1-\aggpar
(1-\Theta)
   s
\right)^j
\\
&=
\left( 
 1-\aggpar
(1-\Theta)
   s
\right)^t 
   (\bD(\alphav^* )-\bD(\vc{\alphav}{0}))
+
\aggpar
(1-\Theta) 
\tfrac{\sigma'}{2\lambda }
(\frac s{  n})^2
4L^2  \sigma 
\frac{1-\left( 
 1-\aggpar
(1-\Theta)
   s
\right)^t}
     { 
  \aggpar
(1-\Theta)
   s }
\\
&\leq
\left( 
 1-\aggpar
(1-\Theta)
   s
\right)^t 
   (\bD(\alphav^* )-\bD(\vc{\alphav}{0}))
+
 s
\frac{4L^2  \sigma   \sigma'}{2\lambda n^2}. 
\tagthis
\label{eq:asfwefcaw}  
 \end{align*}
Choice of 
$s=1$ and $t= t_0:= \max\{0,\lceil  
\frac1{\aggpar (1-\Theta)}
\log(
 2\lambda n^2 (\bD(\alphav^* )-\bD(\vc{\alphav}{0}))
  / (4 L^2 \sigma \sigma')
  )
 \rceil\}$
will lead to 
\begin{align}\label{eq:induction_step1}
  \Exp[\bD(\alphav^*)-\bD(\vc{\alphav}{t})]
 &\leq  
\left( 
 1-\aggpar
(1-\Theta)  
\right)^{t_0}
  (\bD(\alphav^* )-\bD(\vc{\alphav}{0}))
+ 
\frac{4L^2  \sigma   \sigma'}{2\lambda n^2}
\leq 
\frac{4L^2  \sigma   \sigma'}{2\lambda n^2}
+
\frac{4L^2  \sigma   \sigma'}{2\lambda n^2}
=
\frac{4L^2  \sigma   \sigma'}{\lambda n^2}.
\end{align} 
Now, we are going to show that 
\begin{align}
\label{eq:expectationOfDualFeasibility}
\forall t\geq t_0 :  \Exp[\bD(\alphav^* )-\bD(\vc{\alphav}{t})]
&\leq 
\frac{4L^2  \sigma   \sigma'}{\lambda n^2( 1+ \frac12  \aggpar (1-\Theta)  (t-t_0))}.
\end{align}
Clearly, \eqref{eq:induction_step1} implies that \eqref{eq:expectationOfDualFeasibility} holds for $t=t_0$.
Now imagine that it holds for any $t\geq t_0$ then we show that it also has to hold for $t+1$. 
Indeed, using 
\begin{equation}
\label{eq:asdfjoawjdfas}
s=
\frac{1}
 {1+ \frac12 \aggpar (1-\Theta) (t-t_0)} \in [0,1]
\end{equation} 
  we obtain
\begin{align*}
\Exp[
\bD(\alphav^* )-\bD(\vc{\alphav}{t+1})]
&\overset{\eqref{eq:asoifejwofa}
}{\leq}
\left( 
 1-\aggpar
(1-\Theta)
   s
\right) 
   (\bD(\alphav^* )-\bD(\vc{\alphav}{t}))
+
\aggpar
(1-\Theta) 
\tfrac{\sigma'}{2\lambda }
(\frac s{  n})^2
4L^2  \sigma
\\
&\overset{\eqref{eq:expectationOfDualFeasibility}
}{\leq}
\left( 
 1-\aggpar
(1-\Theta)
   s
\right) 
   \frac{4L^2  \sigma   \sigma'}{\lambda n^2( 1+ \frac12  \aggpar (1-\Theta)  (t-t_0))}
+
\aggpar
(1-\Theta) 
\tfrac{\sigma'}{2\lambda }
(\frac s{  n})^2
4L^2  \sigma
\\
&
\overset{\eqref{eq:asdfjoawjdfas}}{=}
\frac{4L^2  \sigma   \sigma'}
     {\lambda n^2}
\left( 
\frac{
1+ \frac12 \aggpar (1-\Theta) (t-t_0)
-\aggpar
(1-\Theta)
+
\aggpar
(1-\Theta) 
\tfrac{1}{2}
}
 {(1+ \frac12 \aggpar (1-\Theta) (t-t_0))^2}
\right)
\\
&=
\frac{4L^2  \sigma   \sigma'}
     {\lambda n^2}
\underbrace{\left( 
\frac{
1+ \frac12 \aggpar (1-\Theta) (t-t_0)
-\frac12 \aggpar
(1-\Theta)
}
 {(1+ \frac12 \aggpar (1-\Theta) (t-t_0))^2}
\right)}_{D}.
\end{align*}
Now, we will upperbound $D$ as follows
\begin{align*}
D&=
\frac1
{1+ \frac12 \aggpar (1-\Theta) (t+1-t_0)}
\underbrace{
\frac{
(1+ \frac12 \aggpar (1-\Theta) (t+1-t_0))
(1+ \frac12 \aggpar (1-\Theta) (t-1-t_0))
}
 {(1+ \frac12 \aggpar (1-\Theta) (t-t_0))^2}}_{\leq 1}
 \\
&\leq  
\frac1
{1+ \frac12 \aggpar (1-\Theta) (t+1-t_0)},
\end{align*}
where in the last inequality we have used the fact that geometric mean
 is less or equal to arithmetic mean. 
 
If $\overline \alphav$ is defined as \eqref{eq:averageOfAlphaDefinition}
then we obtain that
\begin{align*}
\Exp[\gap(\overline\alphav)] &=  
 \Exp\left[\gap\left(\sum_{t=T_0}^{T-1} \tfrac1{T-T_0} \vc{\alphav}{t}\right)\right]
 \leq
  \tfrac1{T-T_0} \Exp\left[\sum_{t=T_0}^{T-1} \gap\left( \vc{\alphav}{t}\right)\right]
\\
&
\overset{
\eqref{eq:lemma:dualDecrease_VS_dualityGap}
,\eqref{eq:asfjoewjofa}
}{\leq}
  \tfrac1{T-T_0} \Exp\left[\sum_{t=T_0}^{T-1} 
\left(
\frac1{\aggpar
(1-\Theta)
 s}
(
\bD(\vc{\alphav}{t+1})
-
\bD(\vc{\alphav}{t})
 )
 +
\tfrac{4L^2 \sigma \sigma' s}{2\lambda n^2 }
\right)  
  \right]
\\  
 &=
\frac1{\aggpar
(1-\Theta)
 s}
   \frac1{T-T_0} 
   \Exp\left[
\bD(\vc{\alphav}{T})
-
\bD(\vc{\alphav}{T_0})
  \right] 
+\tfrac{4L^2 \sigma \sigma' s}{2\lambda n^2 }  
\\  
 &\leq
\frac1{\aggpar
(1-\Theta)
 s}
   \frac1{T-T_0} 
   \Exp\left[
\bD(\alphav^*)
-
\bD(\vc{\alphav}{T_0})
  \right] 
+\tfrac{4L^2 \sigma \sigma' s}{2\lambda n^2 }.  
\tagthis \label{eq:askjfdsanlfas}
  \end{align*}
Now, if $T\geq \lceil
\frac1{\aggpar (1-\Theta)}\rceil+T_0$ such that $T_0\geq t_0$
we obtain
\begin{align*}
\Exp[\gap(\overline\alphav)] 
&\overset{\eqref{eq:askjfdsanlfas}
,\eqref{eq:expectationOfDualFeasibility}
}{\leq}
\frac1{\aggpar
(1-\Theta)
 s}
   \frac1{T-T_0} 
\left(
\frac{4L^2  \sigma   \sigma'}{\lambda n^2( 1+ \frac12  \aggpar (1-\Theta)  (T_0-t_0))}
\right)
+\frac{4L^2 \sigma \sigma' s}{2\lambda n^2 }
\\
&=
\frac{
4L^2  \sigma   \sigma'}{\lambda n^2}
\left(
\frac1{\aggpar
(1-\Theta)
 s}
   \frac1{T-T_0} 
\frac{1}{ 1+ \frac12  \aggpar (1-\Theta)  (T_0-t_0)}
+\frac{  s}{2 }
\right). 
\tagthis
\label{eq:fawefwafewa}
\end{align*}
Choosing 
\begin{equation}
\label{eq:afskoijewofaw}
s=\frac{1}{(T-T_0) \aggpar (1-\Theta)} \in [0,1]
\end{equation}
gives us
\begin{align*}
\Exp[\gap(\overline\alphav)] 
&
\overset{\eqref{eq:fawefwafewa},
\eqref{eq:afskoijewofaw}}{\leq}
\frac{
4L^2  \sigma   \sigma'}{\lambda n^2}
\left(
\frac{1}{ 1+ \frac12  \aggpar (1-\Theta)  (T_0-t_0)}
+\frac{1}{(T-T_0) \aggpar (1-\Theta)} \frac{  1}{2 }
\right). \tagthis
\label{eq:afsjweofjwafea}
\end{align*}
To have right hand side of
\eqref{eq:afsjweofjwafea}
smaller then 
$\epsilon_\gap$
it is sufficient to choose
$T_0$ and $T$ such that
\begin{eqnarray}
\label{eq:sfadwafeewafa}
\frac{4L^2  \sigma   \sigma'}{\lambda n^2}
\left(
\frac{1}{ 1+ \frac12  \aggpar (1-\Theta)  (T_0-t_0)}
\right)
&\leq & \frac12 \epsilon_\gap,
\\
\label{eq:sfadwafeewafa2}
\frac{4L^2  \sigma   \sigma'}{\lambda n^2}
\left(
\frac{1}{(T-T_0) \aggpar (1-\Theta)} \frac{  1}{2 }
\right)
&\leq & \frac12 \epsilon_\gap.
\end{eqnarray}
Hence of 
if
\begin{eqnarray*}
t_0+
\frac{2}{ \aggpar (1-\Theta) }
\left(
\frac
{8L^2  \sigma   \sigma'}
{\lambda n^2 \epsilon_\gap}
-1
\right)
&\leq & 
 T_0 
,
\\
T_0
+
\frac
{4L^2  \sigma   \sigma'}
{\lambda n^2 \epsilon_\gap
\aggpar (1-\Theta)}
&\leq &  T,  
\end{eqnarray*}
then 
\eqref{eq:sfadwafeewafa}
and
\eqref{eq:sfadwafeewafa2}
are satisfied.

\subsection{Proof of Theorem \ref{thm:convergenceSmoothCase}
}
If the function $\ell_i(.)$ is $(1/\mu)$-smooth then $\ell_i^*(.)$ is $\mu$-strongly convex with respect to the
$\|\cdot\|$ norm.
From \eqref{eq:defOfR}
we have
\begin{align*}
\vc{R}{t}&
\overset{\eqref{eq:defOfR}}{=}
-
\tfrac{ \lambda\mu n (1-s)}{\sigma' s }
   \|\vc{\uv}{t}-\vc{\alphav}{t}\|^2 
+ 
 {\sum}_{k=1}^K   
  \| A \vsubset{  (\vc{\uv}{t}  - \vc{\alphav}{t} )}{k}\|^2
\\%
&
\overset{\eqref{eq:definitionOfSigmaK}}{\leq}  
-
\tfrac{ \lambda\mu n (1-s)}{\sigma' s }
   \|\vc{\uv}{t}-\vc{\alphav}{t}\|^2 
+ 
 {\sum}_{k=1}^K   
 \sigma_k
  \|  \vsubset{   \vc{\uv}{t}  - \vc{\alphav}{t}  }{k}\|^2
\\
&\leq
-
\tfrac{ \lambda\mu n (1-s)}{\sigma' s }
   \|\vc{\uv}{t}-\vc{\alphav}{t}\|^2 
+
\sigma_{\max} 
 {\sum}_{k=1}^K   
  \|  \vsubset{   \vc{\uv}{t}  - \vc{\alphav}{t}  }{k}\|^2
\\
&=
\left(
-
\tfrac{ \lambda\mu n (1-s)}{\sigma' s }
+\sigma_{\max}
\right)
   \|\vc{\uv}{t}-\vc{\alphav}{t}\|^2.\tagthis
   \label{eq:afjfocjwfcea} 
\end{align*}
 If we plug 
 \begin{equation}
 s=
  \frac{ \lambda\mu n }
      {\lambda\mu n+
\sigma_{\max} \sigma'}\in [0,1]
\label{eq:fajoejfojew}
\end{equation} 
into
\eqref{eq:afjfocjwfcea}
we obtain that
$\forall t: \vc{R}{t}\leq 0$.
Putting the  same $s$
into
\eqref{eq:lemma:dualDecrease_VS_dualityGap}
will give us
\begin{align*}
&\Exp[
\bD(\vc{\alphav}{t+1})
-
\bD(\vc{\alphav}{t})
 ]
\overset{\eqref{eq:lemma:dualDecrease_VS_dualityGap}
,\eqref{eq:fajoejfojew}}{\geq}
\aggpar
(1-\Theta)
 \frac{ \lambda\mu n }
      {\lambda\mu n+
\sigma_{\max} \sigma'} \gap(\vc{\alphav}{t})
\geq
\aggpar
(1-\Theta)
 \frac{ \lambda\mu n }
      {\lambda\mu n+
\sigma_{\max} \sigma'} \bD(\alphav^*)-\bD(\vc{\alphav}{t}).
\tagthis
\label{eq:fasfawfwaf}
\end{align*}
Using the fact that
$\Exp[\bD(\vc{\alphav}{t+1})-\bD(\vc{\alphav}{t})]
=\Exp[\bD(\vc{\alphav}{t+1})-\bD(\alphav^*)]
+\bD(\alphav^*)-\bD(\vc{\alphav}{t})
$
we have 
\begin{align*}
\Exp[\bD(\vc{\alphav}{t+1})-\bD(\alphav^*)]
+\bD(\alphav^*)-\bD(\vc{\alphav}{t})
\overset{
\eqref{eq:fasfawfwaf}}
{
\geq
}
\aggpar
(1-\Theta)
 \frac{ \lambda\mu n }
      {\lambda\mu n+
\sigma_{\max} \sigma'} \bD(\alphav^*)-\bD(\vc{\alphav}{t})
\end{align*}
which is equivalent with
\begin{align*}
\Exp[\bD(\alphav^*)-\bD(\vc{\alphav}{t+1})]
\leq 
\left(
1-\aggpar
(1-\Theta)
 \frac{ \lambda\mu n }
      {\lambda\mu n+
\sigma_{\max} \sigma'}\right)
\bD(\alphav^*)-\bD(\vc{\alphav}{t}).
\tagthis \label{eq:affpja}
\end{align*}
Therefore if we denote by $\vc{\epsilon_\bD}{t} = \bD(\alphav^*)-\bD(\vc{\alphav}{t})$
we have that
\begin{align*}
 \Exp[\vc{\epsilon_\bD}{t}] 
 \overset{\eqref{eq:affpja}}{\leq}   \left(
 1-\aggpar
(1-\Theta)
 \frac{ \lambda\mu n }
      {\lambda\mu n+
\sigma_{\max} \sigma'}
   \right)^t \vc{\epsilon_\bD}{0}
\overset{\eqref{eq:afjfjaoefvcwa}}{\leq}
\left(
 1-\aggpar
(1-\Theta)
 \frac{ \lambda\mu n }
      {\lambda\mu n+
\sigma_{\max} \sigma'}
   \right)^t
\leq \exp\left(-t \aggpar
(1-\Theta)
 \frac{ \lambda\mu n }
      {\lambda\mu n+
\sigma_{\max} \sigma'}
     \right).
\end{align*}
The right hand side will be smaller than some $\epsilon_\bD$ if 
$$
 t   
    \geq 
\frac{1}
   {\aggpar
(1-\Theta)}
\frac
{\lambda\mu n+
\sigma_{\max} \sigma'}
{ \lambda\mu n }
    \log \frac1{\epsilon_\bD}.
$$
Moreover, to bound the duality gap, we have
\begin{align*}
\aggpar
(1-\Theta)
 \frac{ \lambda\mu n }
      {\lambda\mu n+
\sigma_{\max} \sigma'} \gap(\vc{\alphav}{t})
&
\overset{
\eqref{eq:fasfawfwaf}
}{\leq}
\Exp[
\bD(\vc{\alphav}{t+1})
-
\bD(\vc{\alphav}{t})
 ]
\leq 
\Exp[
\bD(\alphav^*)
-
\bD(\vc{\alphav}{t})
 ]. 
\end{align*}
Therefore  $\gap(\vc{\alphav}{t})\leq 
\frac1{
\aggpar
(1-\Theta)}
 \frac      {\lambda\mu n+
\sigma_{\max} \sigma'} 
{ \lambda\mu n }    \vc{\epsilon_\bD}{t}$.  
Hence if $\epsilon_\bD \leq 
\aggpar
(1-\Theta)
 \frac{ \lambda\mu n }
      {\lambda\mu n+
\sigma_{\max} \sigma'} 
 \epsilon_\gap $
then $\gap(\vc{\alphav}{t})\leq \epsilon_\gap$.
Therefore
after 
$$
 t   
    \geq 
\frac{1}
   {\aggpar
(1-\Theta)}
\frac
{\lambda\mu n+
\sigma_{\max} \sigma'}
{ \lambda\mu n }
    \log 
\left(
\frac{1}
   {\aggpar
(1-\Theta)}
\frac
{\lambda\mu n+
\sigma_{\max} \sigma'}
{ \lambda\mu n }
    \frac1{\epsilon_\gap}
    \right) 
$$
iterations we have obtained a duality gap less than $\epsilon_\gap$.

\subsection{Proof of Theorem \ref{thm:LocalSDCA_smooth2}}

Because $\ell_i$ are $(1/\mu)$-smooth then 
functions
$\ell_i^*$ are $\mu$
strongly convex with respect to the norm $\|\cdot\|$.
The proof is based on
techniques developed in recent coordinate descent papers, including
\cite{richtarik,
richtarik2013distributed,richtarikBigData,TTR:IMPROVRED,
marecek2014distributed,APPROX,lu2013complexity,fercoq2014fast,ALPHA,QUARTZ} (Efficient accelerated variants were considered in \cite{APPROX,  ASDCA}).

First, let us define the
function
$F(\zetav): \R^{n_k} \to \R$
as 
$F(\zetav) := -\Ggk( 
\sum_{i \in \mathcal{P}_k} \zeta_i \ev_i; \wv, \vsubset{\alphav}{k})
$.  This function can be written in two parts
$F(\zetav) = \Phi(\zetav) + f(\zetav)$.
The first part
denoted by 
$\Phi(\zetav)
 =\frac1n\sum_{i \in \mathcal{P}_k} 
\ell_i^*(-\alpha_i - \zeta_i)$
is strongly convex
with convexity parameter
$\frac{\mu}{n}$
with respect to the standard Euclidean norm.
In our application, we think of the $\zetav$ variable collecting the local dual variables $\vsubset{\Delta \alphav}{k}$.

The second part
we will denote by
$f(\zetav)
 = 
  \frac1K 
\frac{\lambda}{2}
\|\wv(\alphav)\|^2
+\frac1n
\sum_{i \in \mathcal{P}_k}
\wv(\alphav)^T \xv_i \zeta_i
+
\frac\lambda2
 \sigma'  
\frac1{\lambda^2 n^2} 
 \| \sum_{i \in \mathcal{P}_k}  \xv_i \zeta_i \|^2 
 $.
It is easy to show
that the gradient of $f$ is coordinate-wise Lipschitz  
 continuous
with Lipschitz constant
$ \frac{\sigma'}{\lambda n^2} r_{\max}$
with respect to the standard Euclidean norm.

Following the
 proof of Theorem 20 in \cite{richtarikBigData}, 
we obtain that
\begin{align*}
\Exp[\Ggk( 
   \vsubset{\Delta \alphav^*}{k}; \wv, \vsubset{\alphav}{k})
-   \Ggk( 
\vc{
  \vsubset{\Delta \alphav}{k}
  }{h+1}; \wv, \vsubset{\alphav}{k})
  ]
&\leq 
\left(
1-\frac1{n_k}
 \frac{1+\frac{\mu n \lambda}{\sigma' r_{\max}}}
      {\frac{\mu n \lambda}{\sigma' r_{\max}}}
\right)
\left(
\Ggk( 
   \vsubset{\Delta \alphav^*}{k};\wv, \vsubset{\alphav}{k})
-   \Ggk( 
\vc{
  \vsubset{\Delta \alphav}{k}
  }{h}; \wv, \vsubset{\alphav}{k})
\right) 
\\
&=
\left(
1-\frac1{n_k}
 \frac
      {    \lambda n \mu }
      {\sigma' r_{\max}+ \lambda n \mu }
\right)
\left(
\Ggk( 
   \vsubset{\Delta \alphav^*}{k}; \wv, \vsubset{\alphav}{k})
-   \Ggk( 
\vc{
  \vsubset{\Delta \alphav}{k}
  }{h}; \wv, \vsubset{\alphav}{k})
\right). 
\end{align*}
Over all steps up to step $h$, this gives
\begin{align*}
\Exp[\Ggk( 
   \vsubset{\Delta \alphav^*}{k}; \wv, \vsubset{\alphav}{k})
-   \Ggk( 
\vc{
  \vsubset{\Delta \alphav}{k}
  }{h}; \wv, \vsubset{\alphav}{k})
  ]
&\leq 
\left(
1-\frac1{n_k}
 \frac
      {    \lambda n \mu }
      {\sigma' r_{\max}+ \lambda n \mu }
\right)^h
\left(
\Ggk( 
   \vsubset{\Delta \alphav^*}{k}; \wv, \vsubset{\alphav}{k})
-   \Ggk({\bf 0}; \wv, \vsubset{\alphav}{k})
\right). 
\end{align*}
 Therefore, choosing 
 $H$ as in the assumption of our Theorem, given in Equation
 \eqref{eq:asjfwjfdwafcea},
 we are guaranteed that
 $\left(
1-\frac1{n_k}
 \frac
      {    \lambda n \mu }
      {\sigma' r_{\max}+ \lambda n \mu }
\right)^H \leq \Theta$, as desired.

\subsection{Proof of Theorem \ref{thm:LocalSDCA_smooth1}
}
Similarly as in the 
proof of Theorem 
\ref{thm:LocalSDCA_smooth2} 
we define a composite function $F(\zetav)
= f(\zetav)+\Phi(\zetav) $.
However, in this case functions
$\ell_i^*$ are not guaranteed to be strongly convex.
However, the first part has still a coordinate-wise Lipschitz continuous gradient with constant
$ \frac{\sigma'}{\lambda n^2} r_{\max}$
with respect to the standard Euclidean norm.
Therefore from Theorem 3 in \cite{TTR:IMPROVRED}
we have that
\begin{align*}
\Exp[\Ggk( 
   \vsubset{\Delta \alphav^*}{k}; \wv, \vsubset{\alphav}{k})
-   \Ggk( 
\vc{
  \vsubset{\Delta \alphav}{k}
  }{h}; \wv, \vsubset{\alphav}{k})
  ]
&\leq 
 \frac{n_k}{n_k+h}
 \left(
 \Ggk( 
   \vsubset{\Delta \alphav^*}{k}; \wv, \vsubset{\alphav}{k})
-   \Ggk( {\bf 0}; \wv, \vsubset{\alphav}{k})
  +\frac12 \frac{\sigma'r_{\max}}{\lambda n^2}  \| \vsubset{\Delta \alphav^*}{k}\|^2
 \right).
 \tagthis
 \label{eq:afewfew}
\end{align*}
 Now, choice 
 of $h=H$ from 
 \eqref{eq:H_convexLoss}
 is sufficient to have
 the right hand side of
 \eqref{eq:afewfew} to be 
 $\leq  
\Theta \big(\Ggk( 
   \vsubset{\Delta \alphav^*}{k}; \wv, \vsubset{\alphav}{k})
-   \Ggk( {\bf 0}; \wv, \vsubset{\alphav}{k}) \big)$.

\section{Relationship of DisDCA to \cocoap}
\label{app:disDCA}
\newcommand{\wlocal}{\uv^{\text{\tiny local}}}
\newcommand{\prev}{{\text{\tiny prev}}}

We are indebted to Ching-pei Lee for showing the following relationship between the practical variant of DisDCA \cite{Yang:2013vl}, and \cocoap when SDCA is chosen as the local solver:

Considering the practical variant of DisDCA (DisDCA-p, see Figure~2 in \cite{Yang:2013vl}) using the scaling parameter $scl=K$, the following holds:

\begin{lemma}\label{lem:equivDisDCA}
Assume that the dataset is partitioned equally between workers,  i.e. $\forall k: n_k = \frac{n}{K}$.
If within the \cocoap framework, SDCA is used as a local solver, and the subproblems are formulated using our shown ``safe'' (but pessimistic) upper bound of $\sigma'=K$, with aggregation parameter $\aggpar=1$ (adding), then the \cocoap framework reduces exactly to the DisDCA-p algorithm. 
\end{lemma}
\begin{proof}(Due to Ching-pei Lee, with some reformulations).
As defined in \eqref{eq:subproblem}, the data-local subproblem solved by each machine in \cocoap is defined as
\[
\max_{\vsubset{\Delta \alphav}{k}\in\R^{n}} \ 
\Ggk(  \vsubset{\Delta \alphav}{k}; \wv, \vsubset{\alphav}{k})
\]
where
\[
\Ggk(  \vsubset{\Delta \alphav}{k}; \wv, \vsubset{\alphav}{k})
\eqdef
-\frac1n\sum_{i \in \mathcal{P}_k} 
\ell_i^*(-\alpha_i - (\vsubset{\Delta \alphav}{k})_i)
- \frac1K 
\frac{\lambda}{2}
\|\wv\|^2
-\frac1n
\wv^T A \vsubset{\Delta \alphav}{k}
- \frac\lambda2
 \sigma'  \Big\|\frac1{\lambda n} A \vsubset{\Delta \alphav}{k}\Big\|^2 \ .
\]
We rewrite the local problem by scaling with $n$, and removing the constant regularizer term $\frac1K \frac{\lambda}{2}\|\wv\|^2$, i.e.
\begin{equation}
\tilde{\Ggk}(  \vsubset{\Delta \alphav}{k}; \wv)
 \eqdef 
-\sum_{j \in \mathcal{P}_k} 
\ell_i^*(-\alpha_j - (\vsubset{\Delta \alphav}{k})_j)
-
\wv^T A \vsubset{\Delta \alphav}{k}
- \frac{\sigma'}{2\lambda n}
   \Big\| A \vsubset{\Delta \alphav}{k}\Big\|^2 \ .
\end{equation}

For the correspondence of interest, we now restrict to single coordinate updates in the local solver. 
In other words, the local solver optimizes exactly one coordinate $i \in \mathcal{P}_k$ at a time.
To relate the single coordinate update to the set of local variables, we will use the notation
\begin{equation}\label{eq:singleCoordNotation}
\vsubset{\Delta \alphav}{k} =: \vsubset{\Delta \alphav^{\prev}}{k} + \delta\ev_i \ ,
\end{equation}
so that $\vsubset{\Delta \alphav^{\prev}}{k}$ are the previous local variables, and $\vsubset{\Delta \alphav}{k}$ will be the updated ones.

From now on, we will consider the special case of \cocoap when the quadratic upper bound parameter is chosen as the ``safe'' value $\sigma'=K$, combined with adding as the aggregation, i.e. $\aggpar=1$.

Now if the local solver within \cocoap is chosen as \localSDCA, then one local step on the subproblem~(\ref{eq:subproblem}) will calculate the following coordinate update. Recall that $A=[\xv_1, \xv_2, \dots, \xv_n] \in \R^{d\times n}$.
\begin{equation}
\delta^\star := \argmax_{\delta\in\R} \ 
  \tilde{\Ggk}(  \vsubset{\Delta \alphav}{k}; \wv)
\end{equation}
which -- because it is only affecting one single coordinate, employing \eqref{eq:singleCoordNotation} -- can be expressed as
\begin{align}
\delta^\star :=& \argmax_{\delta\in\R} \ 
-\ell^*_{i}(-(\alpha_i+(\vsubset{\Delta \alphav^{\prev}}{k})_i+\delta))
-\delta \xv_{i}^T\wv
- \frac{K}{\lambda n} \delta \xv_i^T A \vsubset{\Delta \alphav^{\prev}}{k}
-\frac{K}{2\lambda n}\delta^2 \|\xv_{i}\|_2^2
\notag\\
=& \argmax_{\delta\in\R} \ 
-\ell^*_{i}(-(\alpha_{i}+(\vsubset{\Delta \alphav^{\prev}}{k})_i+\delta))
-\delta \xv_{i}^T \Big( \underbrace{ \wv
+ \frac{K}{\lambda n}  A \vsubset{\Delta \alphav^{\prev}}{k} }_{=:\wlocal} \Big)
-\frac{K}{2\lambda n}\delta^2 \|\xv_{i}\|_2^2
\label{eq:coordupdates}
\end{align}
From this formulation, it is apparent that single coordinate local solvers should maintain their locally updated version of the current primal parameters, which we here denote as
\begin{equation}
\wlocal = \wv + \frac{K}{\lambda n} A\vsubset{\Delta \alphav^{\prev}}{k} \ .
\end{equation}

In the practical variant of DisDCA, the summarized local primal updates are 
$
\Delta \wlocal = \frac{1}{\lambda n_k} A\vsubset{\Delta \alphav}{k}
$.
For the balanced case $n_k = n/K$ for $K$ being the number of machines, this means the local $\wlocal$ update of DisDCA-p is
\begin{align}
\Delta \alpha_i^\star :=& \argmax_{\Delta \alpha_i\in\R} \ 
-\ell^*_{i}(-(\alpha_{i}+\Delta \alpha_i))
-\Delta \alpha_i \xv_{i}^T  \wlocal 
-\frac{K}{2\lambda n}(\Delta \alpha_i)^2 \|\xv_{i}\|_2^2 \ .
\label{eq:disDCAupdates}
\end{align}

It is not hard to show that during one outer round, the evolution of the local dual variables $\vsubset{\Delta \alphav}{k}$ is the same in both methods, such that they will also have the same trajectory of $\wlocal$. This requires some care if the same coordinate is sampled more than once in a round, which can happen in \localSDCA within \cocoap and also in DisDCA-p.
\end{proof}

\paragraph{Discussion.}
In the view of the above lemma, we will summarize the connection of the two methods as follows:

\begin{itemize}
\item \textbf{\cocoa/+ is Not an Algorithm.}
In contrast, it is a framework which allows to use \textit{any local solver} to perform approximate steps on the local subproblem.
This additional level of abstraction (from the definition of such local subproblems in \eqref{eq:subproblem}) is the first to allow \textit{reusability} of any fast/tuned and problem specific single machine solvers, while decoupling this from the distributed algorithmic scheme, as presented in Algorithm \ref{alg:cocoa}.

Concerning the choice of local solver to be used within \cocoa/+,  SDCA is \emph{not} the fastest known single machine solver for most applications.
Much recent research has shown improvements on SDCA \cite{ShalevShwartz:2013wl}, such as accelerated variants \cite{MinibatchASDCA} and other approaches including variance reduction, methods incorporating second-order information, and importance sampling.
In this light, we encourage the user of the \cocoa or \cocoap framework to plug in the best and most recent solver available for their particular local problem (within Algorithm \ref{alg:cocoa}), which is not necessarily SDCA. This choice should be made explicit especially when comparing algorithms.
Our presented convergence theory from Section~\ref{sec:convergence} will still cover these choices, since it only depends on the relative accuracy $\Theta$ of the chosen local solver.

\item \textbf{\cocoap is Theoretically Safe, while still Adaptive to the Data.}
The general definition of the local subproblems, and therefore the treatment of the varying separable bound on the objective -- quantified by $\sigma'$ -- allows our framework to adapt to the difficulty of the data partition and still give convergence results. 
The data-dependent measure $\sigma'$ is fully decoupled from what the user of the framework prefers to employ as a local solver (see also the comment below that \cocoa is not a coordinate solver).

The safe upper bound $\sigma'=K$ is worst-case pessimistic, for the convergence theory to still hold in all cases, when the updates are added.
Using additional knowledge from the input data, better bounds and therefore better step-sizes can be achieved in \cocoap.
An example when $\sigma'$ can be safely chosen much smaller is when the data-matrix satisfies strong row/column sparsity, see e.g. Lemma 1 in \cite{richtarik2013distributed}.

\item \textbf{Obtaining DisDCA-p as a Special Case.}
As shown in Lemma \ref{lem:equivDisDCA} above, we have that if in \cocoap, if SDCA is used as the local solver
and the pessimistic upper bound of $\sigma'=K$ is used
and, moreover, the dataset is partitioned equally, 
i.e. $\forall k: n_k = \frac{n}{K}$,
then the \cocoap framework reduces exactly to the DisDCA-p algorithm by  \cite{Yang:2013vl}. 

The correspondence breaks down if the subproblem parameter is chosen to a practically good value $\sigma'\ne K$. Also, as noted above, SDCA is often not the best local solver currently available. In our above experiments, SDCA was used just for demonstration purposes and ease of comparison. Furthermore, the data partition might often be unbalanced in practical applications.

While both DisDCA-p and \cocoa are special cases of \cocoap, we note that DisDCA-p can not be recovered as a special case of the original \cocoa framework \cite{jaggi2014communication}.

\item \textbf{\cocoa/+ are Not Coordinate Methods.}
Despite the original name being motivated from this special case,  \cocoa and \cocoap are \emph{not} coordinate methods. In fact, \cocoap as presented here for the adding case ($\aggpar = 1$) is much more closely related to a batch method applied to the dual, using a block-separable proximal term, as following from our new subproblem formulation \eqref{eq:subproblem}, depending on $\sigma'$. See also the remark in Section \ref{sec:relatedWork}.
The framework here (Algorithm \ref{alg:cocoa}) gives more generality, as the used local solver is not restricted to be a coordinate-wise one. In fact the framework allows to translate recent and future improvements of single machine solvers directly to the distributed setting, by employing them within Algorithm \ref{alg:cocoa}.
DisDCA-p works very well for several applications, but is restricted to using local coordinate ascent (SDCA) steps.

\item \textbf{Theoretical Convergence Results.}
While DisDCA-p \cite{Yang:2013vl} was proposed without theoretical justification (hence the nomenclature), the main contribution in the paper here -- apart from the arbitrary local solvers -- is the convergence analysis for the framework.
The theory proposed in \cite{Yang:2013ui} is given only for the setting of orthogonal partitions, i.e., when $\sigma'=1$ and the problems become trivial to distribute given the orthogonality of data between the workers.

The theoretical analysis here gives convergence rates applying for Algorithm \ref{alg:cocoa} when using arbitrary local solvers, and inherits the performance of the local solver.
As a special case, we obtain the first theoretical justification and convergence rates for original \cocoa in the case of general convex objective, as well as for the special case of DisDCA-p for both general convex and smooth convex objectives.

\end{itemize}

\end{document}